\documentclass[twoside]{article}

\usepackage[accepted]{aistats2024}

\usepackage[round]{natbib}

\usepackage{graphicx}
\usepackage{amsfonts}
\usepackage[pagecolor=none,dvipsnames]{xcolor}
\usepackage{hyperref}
\usepackage{cleveref}
\usepackage{autonum}
\usepackage{multirow}
\usepackage{algorithm}
\usepackage{algpseudocode}
\usepackage{listings}
\usepackage{amsthm}
\usepackage{bm}

\newcommand{\R}{\mathbb R}
\newcommand{\E}{\mathbb E}
\newcommand{\N}{\mathcal N}
\newcommand{\KL}{\mathcal{D}_\text{KL}}
\newcommand{\tr}{\operatorname{tr}}
\newcommand{\loss}{\mathcal{L}}

\newtheorem{theorem}{Theorem}[section]
\newtheorem{corollary}{Corollary}[theorem]

\definecolor{C0}{HTML}{1f77b4}
\definecolor{C1}{HTML}{ff7f0e}
\definecolor{C2}{HTML}{2ca02c}
\definecolor{C3}{HTML}{d62728}
\definecolor{C4}{HTML}{9467bd}
\definecolor{C5}{HTML}{8c564b}
\definecolor{C6}{HTML}{e377c2}
\definecolor{C7}{HTML}{7f7f7f}
\definecolor{C8}{HTML}{bcbd22}
\definecolor{C9}{HTML}{17becf}

\crefname{lstlisting}{listing}{listings}
\Crefname{lstlisting}{Listing}{Listings}

\definecolor{ipython_keyword}{RGB}{0,128,0}
\definecolor{ipython_builtin}{RGB}{0,128,0}
\definecolor{ipython_comment}{RGB}{61,123,123}
\definecolor{ipython_string}{RGB}{186,33,33}

\makeatletter
\newcommand\unnumberedfootnote[1]{%
  \begingroup
  \let\@makefnmark\relax  %
  \let\@thefnmark\relax   %
  \footnotetext{#1}%
  \endgroup
}
\makeatother

\begin{document}

\runningauthor{ Felix Draxler, Peter Sorrenson, Lea Zimmermann, Armand Rousselot, Ullrich Köthe }

\twocolumn[

\aistatstitle{Free-form Flows: Make Any Architecture a Normalizing Flow}

\aistatsauthor{ Felix Draxler*, Peter Sorrenson*, Lea Zimmermann, Armand Rousselot, Ullrich Köthe }

\aistatsaddress{ Computer Vision and Learning Lab, Heidelberg University\\**Equal contribution } ]

\begin{abstract}

    Normalizing Flows are generative models that directly maximize the likelihood. Previously, the design of normalizing flows was largely constrained by the need for analytical invertibility. We overcome this constraint by a training procedure that uses an efficient estimator for the gradient of the change of variables formula. This enables any dimension-preserving neural network to serve as a generative model through maximum likelihood training. Our approach allows placing the emphasis on tailoring inductive biases precisely to the task at hand. Specifically, we achieve excellent results in molecule generation benchmarks utilizing $E(n)$-equivariant networks at greatly improved sampling speed. Moreover, our method is competitive in an inverse problem benchmark, while employing off-the-shelf ResNet architectures. We publish our code at \url{https://github.com/vislearn/FFF}.
\end{abstract}

\section{INTRODUCTION}

\begin{figure}[th]
    \centering
    \includegraphics[width=\linewidth]{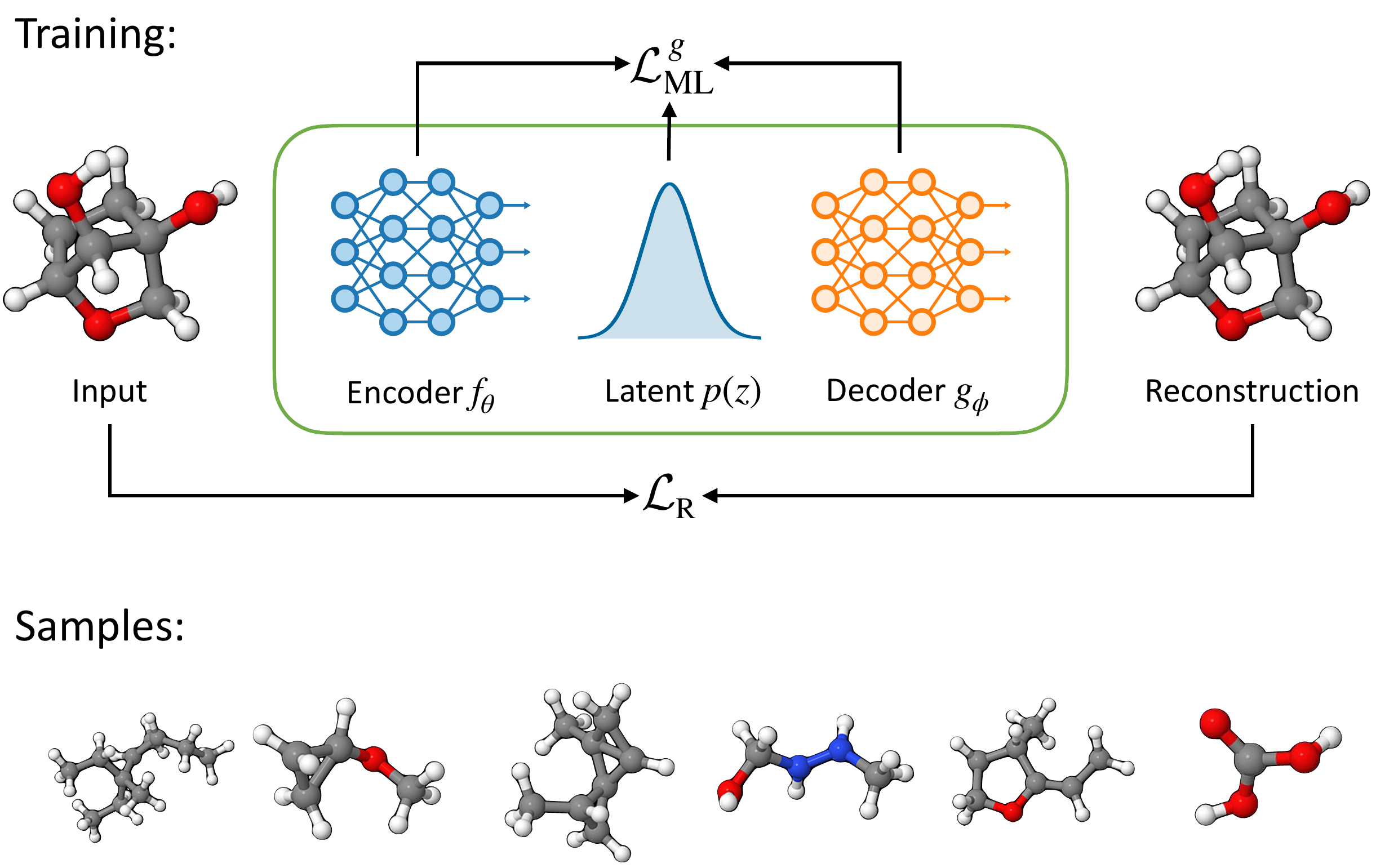}
    \caption{Free-form flows (FFF) train a pair of encoder and decoder neural networks with a fast maximum likelihood estimator~$\loss_{\operatorname{ML}}^{g}$ and reconstruction loss~$\loss_{\operatorname{R}}$. This enables training any dimension-preserving architecture as a one-step generative model. For example, an equivariant graph neural network can be trained on the QM9 dataset to generate molecules by predicting atom positions and properties in a single decoder evaluation. \textit{(Bottom)} Stable molecules sampled from our $E(3)$-FFF trained on the QM9 dataset for several molecule sizes.}
    \label{fig:molecule-qm9}
\end{figure}

Generative models have actively demonstrated their utility across diverse applications, successfully scaling to high-dimensional data distributions in scenarios ranging from image synthesis to molecule generation \citep{rombach2022high, hoogeboom2022equivariant}. Normalizing flows \citep{dinh2015nice, rezende2015variational} have helped propel this advancement, particularly in scientific domains, enabling practitioners to optimize data likelihood directly and thereby facilitating a statistically rigorous approach to learning complex data distributions. A major factor that has held normalizing flows back as other generative models (notably diffusion models) increase in power and popularity has been that their expressivity is greatly limited by architectural constraints, namely those necessary to ensure bijectivity and compute Jacobian determinants.

In this work, we contribute an approach that frees normalizing flows from their conventional architectural confines, thereby introducing a flexible new class of maximum likelihood models. For model builders, this shifts the focus away from meeting invertibility requirements towards incorporating the best inductive biases to solve the problem at hand. Our aim is that the methods introduced in this paper will allow practitioners to spend more time incorporating domain knowledge into their models, and allow more problems to be solved via maximum likelihood estimation.

The key methodological innovation is the adaptation of a recently proposed method for training autoencoders \citep{sorrenson2024lifting} to dimension-preserving models. The trick is to estimate the gradient of the encoder's Jacobian log-determinant by a cheap function of the encoder and decoder Jacobians. We show that in the full-dimensional context many of the theoretical difficulties that plagued the interpretation of the bottlenecked autoencoder model disappear, and the optimization can be interpreted as a relaxation of normalizing flow training, which is tight at the original solutions. 

In molecule generation, where rotational equivariance has proven to be a crucial inductive bias, our approach outperforms traditional normalizing flows and generates valid samples more than an order of magnitude faster than previous approaches. Further, experiments in simulation-based inference (SBI) underscore the model’s versatility. We find that our training method achieves competitive performance with minimal fine-tuning requirements.

In summary our contributions are as follows:
\begin{itemize}
    \item We remove all architectural constraints from normalizing flows by introducing maximum-likelihood training for free-form architectures. We call our model the free-form flow (FFF), see \cref{fig:molecule-qm9,sec:method}.
    \item We prove that the training has the same minima as traditional normalizing flow optimization, provided that the reconstruction loss is minimal, see \cref{sec:theory}.
    \item We demonstrate competitive performance with minimal fine-tuning on inverse problems and molecule generation benchmarks, outperforming ODE-based models in the latter. Compared to a diffusion model, our model produces stable molecules more than two orders of magnitude faster. See \cref{sec:experiments}.

\end{itemize}

\section{RELATED WORK}

Normalizing flows traditionally rely on specialized architectures that are invertible and have a manageable Jacobian determinant (see \cref{sec:normalizing-flows}). See \cite{papamakarios2021normalizing, kobyzev2021normalizing} for an overview.

One body of work builds invertible architectures by concatenating simple layers (coupling blocks) which are easy to invert and have a triangular Jacobian, which makes computing determinants easy \citep{dinh2015nice}. Expressive power is obtained by stacking many layers and their universality has been confirmed theoretically \citep{huang2020augmented, teshima2020couplingbased, koehler2021representational, draxler2022whitening, draxler2023convergence}. Many choices for coupling blocks have been proposed such as MAF \cite{papamakarios2017masked}, RealNVP \citep{dinh2017density}, Glow \citep{kingma2018glow}, Neural Spline Flows \citep{durkan2019neural}, see \cite{kobyzev2021normalizing} for an overview. Instead of analytical invertibility, our model relies on the reconstruction loss to enforce approximate invertibility.

Another line of work ensures invertibility by using a ResNet structure and limiting the Lipschitz constant of each residual layer \citep{behrmann2019invertible, chen2019residual}. Somewhat similarly, neural ODEs \citep{chen2018neural, grathwohl2019ffjord} take the continuous limit of ResNets, guaranteeing invertibility under mild conditions. Each of these models requires evaluating multiple steps during training and thus become quite expensive. In addition, the Jacobian determinant must be estimated, adding overhead. Like these methods, we must estimate the gradient of the Jacobian determinant, but can do so more efficiently. Flow Matching \cite{lipman2023flow, liu2023learning, albergo2023building} improves training of these continuous normalizing flows in speed and quality, but still involves an expensive multi-step sampling process. By construction, our approach consists of a single model evaluation, and we put no constraints on the architecture apart from inductive biases indicated by the task at hand.

Two interesting methods \citep{gresele2020relative, keller2021self} compute or estimate gradients of the Jacobian determinant but are severely limited to architectures with exclusively square weight matrices and no residual blocks. We have no architectural limitations besides preserving dimension. Intermediate activations and weight matrices may have any dimension and any network topology is permitted.

\section{METHOD}
\label{sec:method}

\subsection{Normalizing Flows}
\label{sec:normalizing-flows}

Normalizing flows \citep{rezende2015variational} are generative models that learn an invertible function $f_\theta(x): \mathbb R^D \to \mathbb R^D$ mapping samples $x$ from a given data distribution $q(x)$ to latent codes $z$. The aim is that $z$ follows a simple target distribution, typically the multivariate standard normal. 

Samples from the resulting generative model $p_\theta(x)$ are obtained by mapping samples of the simple target distribution $p(z)$ through the inverse of the learned function:
\begin{equation}
    x = f_\theta^{-1}(z) \sim p_\theta(x) \text{ for } z \sim p(z).
\end{equation}
This requires a tractable inverse. Traditionally, this was achieved via invertible layers such as coupling blocks \citep{dinh2015nice} or by otherwise restricting the function class. We replace this constraint via a simple reconstruction loss, and learn a second function $g_\phi \approx f_\theta^{-1}$ as an approximation to the exact inverse.

A tractable determinant of the Jacobian of the learned function is required to account for the change in density. As a result, the value of the model likelihood is given by the change of variables formula for invertible functions:
\begin{equation}
    \label{eq:cov}
    p_\theta(x) = p(Z=f_\theta(x)) |J_\theta(x)|.
\end{equation}
Here, $J_\theta(x)$ denotes the Jacobian of $f_\theta$ at $x$, and $|\cdot|$ the absolute value of its determinant.

Normalizing Flows are trained by minimizing the Kullback-Leibler (KL) divergence between the true and learned distribution. This is equivalent to maximizing the likelihood of the training data:
\begin{align} &
    \KL(q(x) \| p_\theta(x))
    = \mathbb E_{x \sim q(x)}[\log q(x) - \log p_\theta(x)] \\
    &= \mathbb E_{x}[ - \log p(f_\theta(x)) - \log |J_\theta(x)|] + \operatorname{const}.
    \label{eq:maximum-likelihood}
\end{align}
By \cref{eq:cov}, this requires evaluating the determinant of the Jacobian $|J_\theta(x)|$ of $f_\theta$ at $x$. If we want to compute this exactly, we need to compute the full Jacobian matrix, requiring $D$ backpropagations through $f_\theta$. This linear scaling with dimension is prohibitive for most modern applications. The bulk of the normalizing flow literature is therefore concerned with building invertible architectures that are expressive and allow computing the determinant of the Jacobian more efficiently. We circumvent this via a trick that allows efficient estimation of the gradient $\nabla_\theta \log |J_\theta(x)|$, noting that this quantity is sufficient to perform gradient descent.

\subsection{Gradient trick}
\label{sec:gradient-trick}

The results of this section are an adaptation of results in \cite{caterini2021rectangular} and \cite{sorrenson2024lifting}.

Here, we derive how to efficiently estimate the gradient of the maximum-likelihood loss in \cref{eq:maximum-likelihood}, even if the architecture does not yield an efficient way to compute the change of variables term $\log|J_\theta(x)|$. 
We avoid this computation by estimating the gradient of $\log|J_\theta(x)|$ via a pair of vector-Jacobian and Jacobian-vector products, which are readily available in standard automatic differentiation software libraries.

\paragraph{Gradient via trace estimator} 
\begin{theorem}
    Let $f_\theta: \R^D \to \R^D$ be a $C^1$ invertible function parameterized by $\theta$. Then, for all $x \in \R^D$:
    \begin{equation}
        \nabla_{\theta_i} \log|J_\theta(x)| = \tr\!\left(
            (\nabla_{\theta_i} J_\theta(x))
            (J_\theta(x))^{-1}
        \right).
        \label{eq:gradient-trick}
    \end{equation}
\end{theorem}
The proof is by direct application of Jacobi's formula, see \cref{app:gradient-via-trace}. This is not a simplification per se, given that the RHS of \cref{eq:gradient-trick} now involves the computation of both the Jacobian as well as its inverse. However, we can estimate it via the Hutchinson trace estimator (where we omit dependence on $x$ for simplicity):
\begin{align}
    \tr\!\left(
        (\nabla_{\theta_i} J_\theta) 
        J_\theta^{-1}
    \right) &
    =
    \E_{v}[v^T (\nabla_{\theta_i} J_\theta) J_\theta^{-1} v] \\&
    \approx \frac{1}{K} \sum_{k=1}^K v_k^T (\nabla_{\theta_i} J_\theta) J_\theta^{-1}  v_k.
\end{align}
Now all we require is computing the dot products $v^T (\nabla_{\theta_i} J_\theta)$ and $J_\theta^{-1} v$, where the random vector $v \in \R^D$ must have unit covariance.

\paragraph{Matrix inverse via function inverse} To compute $J_\theta^{-1} v$ we note that, when $f_\theta$ is invertible, the matrix inverse of the Jacobian of $f_\theta$ is the Jacobian of the inverse function $f_\theta^{-1}$:
\begin{equation}
    J_\theta^{-1}(x) = (\nabla_x f_\theta(x))^{-1} = \nabla_z f_\theta^{-1}\big(z = f_\theta(x)\big).
\end{equation}

This means that the product $J_\theta^{-1} v$ is simply the dot product of the Jacobian of $f_\theta^{-1}$ with the vector $v$. This Jacobian-vector product is readily available via forward automatic differentiation.

\paragraph{Use of stop-gradient} We are left with computing the dot product $v^T (\nabla_{\theta_i} J_\theta)$. Since $v^T$ is independent of $\theta$, we can draw it into the gradient $v^T (\nabla_{\theta_i} J_\theta)  = \nabla_{\theta_i} (v^T J_\theta)$. This vector-Jacobian product can be again readily computed, this time with backward automatic differentiation.

In order to implement the final gradient with respect to the flow parameters $\theta$, we draw the derivative with respect to parameters out of the trace, making sure to prevent gradient from flowing to $J_\theta^{-1}$ by wrapping it in a stop-gradient operation \texttt{SG}:
\begin{align}
    \tr\!\left(
        (\nabla_{\theta_i} J_\theta)
        J_\theta^{-1}
    \right) &
    =
    \nabla_{\theta_i} \tr(
        J_\theta
        \mathtt{SG}(J_\theta^{-1})
    ) \\&
    \approx
    \nabla_{\theta_i} \frac{1}{K} \sum_{k=1}^K
        v_k^T J_\theta 
        \mathtt{SG}(J_\theta^{-1} v_k).
\end{align}

\paragraph{Summary} The above argument shows that
\begin{equation}
    \nabla_{\theta_i} \log|J_\theta(x)| 
    \approx
    \nabla_{\theta_i} \frac{1}{K} \sum_{k=1}^K
        v_k^T J_\theta 
        \mathtt{SG}(J_\theta^{-1} v_k).
    \label{eq:change-of-variables-gradient}
\end{equation}
Instead of computing the full Jacobian $J_\theta(x)$, which involved as many backpropagation steps as dimensions, we are left with computing just one vector-Jacobian product 
and one Jacobian-vector product 
for each $k$.
In practice, we find that setting $K=1$ is sufficient and we drop the summation over $k$ for the remainder of this paper. We provide an ablation study on the effect of $K$ in \cref{app:ablation}.

This yields the following maximum likelihood training objective, whose gradients are an unbiased estimator for the true gradients from exact maximum likelihood as in \cref{eq:maximum-likelihood}:
\begin{equation}
    \label{eq:loss-ML-f-inv}
    \loss_{\text{ML}}^{f^{-1}} = \E_{x,v}[-\log p(f_\theta(x)) -  v^T J_\theta \mathtt{SG}(J_\theta^{-1} v) ].
\end{equation}
This result enables training normalizing flow architectures with a tractable inverse function, but whose Jacobian determinant is not easily accessible.
We now move on to show how this gradient estimator can be adapted for free-form dimension-preserving neural networks.

\subsection{Free-form Flows (FFF)}
\label{sec:fff}

\begin{algorithm}[t]
\caption{FFF loss function. Vector-Jacobian product = \texttt{vjp}; Jacobian-vector product = \texttt{jvp}. Time and space complexity are $\mathcal{O}(D)$.}
\label{alg:loss}
\begin{algorithmic}

\Function{Loss}{$x, f_\theta, g_\phi, \beta$}
    \State $v \sim p(v)$\Comment{$\E[vv^T] = I$}
    \State $z, v_1 \gets \texttt{vjp}(f_\theta, x, v)$\Comment{$z = f_\theta(x), v_1^T = v^T \frac{\partial z}{\partial x}$}
    \State $\hat x, v_2 \gets \texttt{jvp}(g_\phi, z, v)$\Comment{$\hat x = g_\phi(z), v_2 = \frac{\partial \hat x}{\partial z} v$}
    \State $\loss_{\text{ML}}^g \gets \frac{1}{2} \lVert z \rVert^2 - \texttt{SG}(v_2^T) v_1$\Comment{stop gradient to $v_2$}
    \State $\loss_\text{R} \gets \lVert \hat x - x \rVert^2$
    \State $\loss^g \gets \loss_\text{ML}^g + \beta \loss_\text{R}$
    \State \textbf{return} $\loss^g$
\EndFunction
\end{algorithmic}
\end{algorithm}

The previous section assumed that we have access to both $f_\theta$ and its analytic inverse $f_\theta^{-1}$. Typically an analytic inverse is obtained by explicitly constructing invertible neural networks (INNs) or defining the flow as a differential equation with a known reverse time process (Neural ODEs). In contrast, we drop the assumption of an analytic inverse and replace $f_\theta^{-1}$ with a learned inverse $g_\phi \approx f_\theta^{-1}$. We ensure that (i) $f_\theta$ is invertible and that (ii) $g_\phi \approx f_\theta^{-1}$ through a reconstruction loss:
\begin{equation}
    \label{eq:loss-R}
    \loss_\text{R} = \tfrac12 \mathbb E_{x}[\| x - g_\phi(f_\theta(x)) \|^2].
\end{equation}
This removes all architectural constraints from $f_\theta$ and $g_\phi$ except from preserving the dimension.

Similarly to \cite{sorrenson2024lifting}, the replacement $g_\phi \approx f_\theta^{-1}$ leads to a modification of $\loss_\text{ML}^{f^{-1}}$, where we replace $J_\theta^{-1}$ by $J_\phi$, where $J_\phi$ is shorthand for the Jacobian of $g_\phi$ evaluated at $f_\theta(x)$:
\begin{equation}
    \label{eq:loss-ML-g}
    \loss_{\text{ML}}^g = \E_{x,v}[-\log p(f_\theta(x)) - v^T J_\theta \mathtt{SG}(J_\phi v) ]
\end{equation}
Combining the maximum likelihood (\cref{eq:loss-ML-f-inv,eq:loss-ML-g}) and reconstruction (\cref{eq:loss-R}) components of the loss leads to the following losses:
\begin{equation}
    \label{eq:loss-f-inv-and-g}
    \loss^{f^{-1}} = \loss_\text{ML}^{f^{-1}} + \beta \loss_\text{R} \quad \text{and} \quad \loss^g = \loss_\text{ML}^g + \beta \loss_\text{R}
\end{equation}
where the two terms are traded off by a hyperparameter $\beta$. We optimize $\loss^g$ with the justification that it has the same critical points as $\loss^{f^{-1}}$ (plus additional ones which aren't a problem in practice, see \cref{sec:critical-points}).

\subsubsection{Likelihood Calculation}

Once training is completed, our generative model involves sampling from the latent distribution and passing the samples through the decoder $g_\phi$.

In order to calculate the likelihoods induced by $g_\phi$, we can use the change of variables formula:
\begin{align}
    p_\phi(X = x) &
    = p(Z = g_\phi^{-1}(x)) |J_\phi(g_\phi^{-1}(x))| \\&
    \approx p(Z = f_\theta(x)) |J_\phi(f_\theta(x))| 
\end{align}
where the approximation is due to $g_\phi^{-1} \approx f_\theta$.

\begin{algorithm}[t]
\caption{FFF likelihood calculation: returns an approximation of $\log p_\phi(x)$. Time complexity is $\mathcal{O}(D^3)$ and space complexity is $\mathcal{O}(D^2)$.}
\label{alg:likelihood}
\begin{algorithmic}
\Function{LogLikelihood}{$x, f_\theta, g_\phi$}
    \State $z \gets f_\theta(x)$
    \State $J_\phi \gets \texttt{jacobian}(g_\phi, z)$\Comment{$J_\phi = \frac{\partial g_\phi(z)}{\partial z}$}
    \State $\ell \gets -\frac{1}{2} \lVert z \rVert^2 - \frac{D}{2}\log(2\pi) + \log|J_\phi|$
    \State \textbf{return} $\ell$
\EndFunction
\end{algorithmic}
\end{algorithm}

In the next section, we theoretically justify the use of free-form architectures and the combination of maximum likelihood with a reconstruction loss.

\section{THEORY}
\label{sec:theory}

In this section we provide three theorems which emphasize the validity of our method. Firstly, we show that optimizing the loss function $\loss^{f^{-1}}$ (using an exact inverse) is a bound on the spread divergence between the data and generating distributions. Secondly, we show under what conditions the gradients of the relaxation $\loss^g$ (loss using a non-exact inverse) equal those of $\loss^{f^{-1}}$. Finally and most importantly, we show that solutions to $\loss^{f^{-1}}$ exactly learn the data distribution. In addition, every critical point of $\loss^{f^{-1}}$ is a critical point of $\loss^g$, meaning that optimizing $\loss^g$ is equivalent, except for some additional critical points which we argue do not matter in practice. Please refer to \cref{app:theory} for detailed derivations and proofs of the results in this section.

\subsection{Loss Derivation}
\label{sec:loss-derivation}

In addition to the intuitive development given in the previous sections, $\loss^{f^{-1}}$ (\cref{eq:loss-f-inv-and-g}) can be rigorously derived as a bound on the KL divergence between a noisy version of the data and a noisy version of the model, known as a spread divergence \citep{zhang2020spread}. The bound is a type of evidence lower bound (ELBO) as employed in VAEs \citep{kingma2013auto}.
\begin{theorem}
Let $f_\theta$ and $g_\phi$ be $C^1$ and let $f_\theta$ be globally invertible. Define the spread KL divergence $\tilde{\mathcal{D}}_\text{KL}$ as the KL divergence between distributions convolved with isotropic Gaussian noise of variance $\sigma^2$. Let $\beta = 1/2\sigma^2$. Then there exists a function $D$ of $\theta$ and $\phi$ such that
\begin{equation}
    \nabla_\theta \loss^{f^{-1}} = \nabla_\theta D \quad \text{and} \quad \nabla_\phi \loss^{f^{-1}} = \nabla_\phi D
\end{equation}
and
\begin{equation}
     D \geq \tilde{\mathcal{D}}_\text{KL}( q(x) \parallel p_\phi(x) )
\end{equation}
As a result, minimizing $\loss^{f^{-1}}$ is equivalent to minimizing an upper bound on $\tilde{\mathcal{D}}_\text{KL}( q(x) \parallel p_\phi(x) )$. 

\end{theorem}
The fact that we optimize a bound on a spread KL divergence is beneficial in cases where $q(x)$ is degenerate, for example when $q(x)$ is an empirical data distribution (essentially a mixture of Dirac delta distributions). (Non-spread) KL divergences with $q(x)$ in this case will be almost always infinite. In addition, by taking $\sigma$ very small (and hence using large $\beta$), the difference between the standard and spread KL divergence is so small as to be negligible in practice.

Since the above derivation resembles an ELBO, we can ask whether the FFF can be interpreted as a VAE. In \cref{app:loss-derivation} we provide an argument that it can, but one with a very flexible posterior distribution, in contrast to the simple distributions (such as Gaussian) typically used in VAEs posteriors. As such it does not suffer from typical VAE failure modes, such as poor reconstructions and over-regularization.

Note that the above theorem states a result in terms of the decoder distribution $p_\phi(x)$, not the encoder distribution $p_\theta(x)$ which was used to motivate the loss function. While this may seem counterproductive at first, it is in fact more useful to optimize for $p_\phi(x)$ matching the data distribution than for $p_\theta(x)$ to match the data distribution, since $p_\phi(x)$ is the model we use to generate data from. In any case, it is simple to demonstrate that $D' \geq \KL(q(x) \parallel p_\theta(x))$ where $D'$ has the same gradients as $\loss^{f^{-1}}$ (see \cref{app:loss-derivation}), so both encoder and decoder models will become more similar to the data distribution with increasing optimization.

\subsection{Error Bound}
\label{sec:error-bound}
The accuracy of the estimator for the gradient of the log-determinant depends on how close $g_\phi$ is to being an inverse of $f_\theta$. In particular, we can bound the estimator's error by a measure of how close the product of the Jacobians is to the identity matrix. This is captured in the following result.
\begin{theorem}
Let $f_\theta$ and $g_\phi$ be $C^1$, let $J_\theta$ be the Jacobian of $f_\theta$ at $x$ and let $J_\phi$ be the Jacobian of $g_\phi$ at $f_\theta(x)$. Suppose that $f_\theta$ is locally invertible at $x$, meaning $J_\theta(x)$ is an invertible matrix. Let $\lVert \cdot \rVert_F$ be the Frobenius norm of a matrix. Then the absolute difference between $\nabla_{\theta_i} \log |J_\theta(x)|$ and the trace-based approximation is bounded:
\begin{align}
    & \hspace{0.45cm} \left| \tr((\nabla_{\theta_i} J_\theta) J_\phi) - \nabla_{\theta_i} \log |J_\theta| \right| \\ & \leq \lVert (\nabla_{\theta_i} J_\theta) J_\theta^{-1} \rVert_F \lVert J_\theta J_\phi  - I \rVert_F
\end{align}
\end{theorem}

The Jacobian deviation, namely $\lVert J_\theta J_\phi - I \rVert_F$, could be minimized by adding such a term to the loss as a regularizer. We find in practice that the reconstruction loss alone is sufficient to minimize this quantity and that the two are correlated in practice. While it could be possible in principle for a dimension-preserving pair of encoder and decoder to have a low reconstruction loss while the Jacobians of encoder and decoder are not well matched, we don't observe this in practice. Such a function would have to have a very large second derivative, which is implicitly discouraged in typical neural network optimization \citep{rahaman2019spectral}. 

In \cref{app:error-bound} we prove an additional result which quantifies the difference between the gradients of $\loss^{f^{-1}}$ and $\loss^g$, showing that it is bounded by $\E_{x} \left[ \lVert J_\theta J_\phi  - I \rVert_F^2 \right]$.

\subsection{Critical Points}
\label{sec:critical-points}

\begin{figure}
    \centering
    \includegraphics[width=1\linewidth]{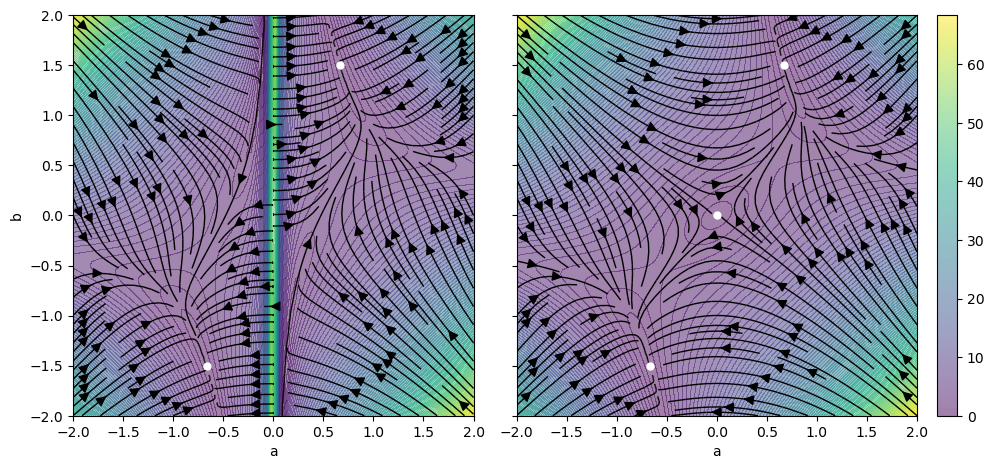}
    \caption{Gradient landscape of $\loss^{f^{-1}}$ (\textit{left}) and $\loss^g$ (\textit{right}) for a linear 1D model with $f(x) = ax$, $g(z) = bz$, $q(x) = \N(0, 1.5^2)$ and $\beta = 1$. The flow lines show the direction and the contours show the magnitude of the gradient. White dots are critical points. $\loss^g$ has the same minima $(\pm 2/3, \pm 1.5)$ as $\loss^{f^{-1}}$, with an additional critical point at $a=b=0$. This is a saddle, so we will not converge to it in practice. Therefore optimizing $\loss^g$ results in the same solutions as $\loss^{f^{-1}}$.}
    \label{fig:critical-points}
\end{figure}

The following theorem states our main result: that optimizing $\loss^g$ (\cref{eq:loss-f-inv-and-g}) is almost equivalent to optimizing $\loss^{f^{-1}}$, and that the solutions to $\loss^{f^{-1}}$ are maximum likelihood solutions where $p_\theta(x) = q(x)$. Note that this a result on the functional level: if we say $f$ is a critical point of $\loss^{f^{-1}}$, we mean that adding any infinitesimally small deviation $\delta f$ to $f$ does not change $\loss^{f^{-1}}$. These optima may not be within the set of functions reachable under gradient descent with our chosen network architecture, and the particular neural network implementation may introduce local minima which are not captured in the theorem.

\newpage

\begin{theorem}
\label{thm:critical-points}
Let $f_\theta$ and $g_\phi$ be $C^1$ and let $f_\theta$ be globally invertible. Suppose $q(x)$ is finite and has support everywhere. Then the critical points (on the functional level) of $\loss^{f^{-1}}$ (for any $\beta > 0$) are such that 
\begin{enumerate}
    \item $g_\phi(z) = f_\theta^{-1}(z)$ for all $z$, and
    \item $p_\theta(x) = q(x)$ for all $x$, and
    \item All critical points are global minima
\end{enumerate}

Furthermore, every minimum of $\loss^{f^{-1}}$ is a critical point of $\loss^g$. If the reconstruction loss is minimal, $\loss^g$ has no additional critical points.
\end{theorem}
Note that $\loss^g$ may have additional critical points if the reconstruction loss is not minimal, meaning that $f_\theta$ and $g_\phi$ are not globally invertible. An example is when both $f_\theta$ and $g_\phi$ are the zero function and $q(x)$ has zero mean. We can avoid such solutions by ensuring that $\beta$ is large enough to not tolerate a high reconstruction loss. In \cref{app:reconstruction-weight} we give guidelines on how to choose $\beta$ in practice.

\Cref{fig:critical-points} provides an illuminating example. Here the data and latent space are 1-dimensional and $f$ and $g$ are simple linear functions of a single parameter each. As such we can visualize the gradient landscape in a 2D plot. We see that the additional critical point at the origin is a saddle: there are both converging and diverging gradients. In stochastic gradient descent, it is not plausible that we converge to a saddle since the set of points which converge to it deterministically have measure zero in the parameter space. Hence in this example $\loss^g$ will converge to the same solutions as $\loss^{f^{-1}}$. In addition, it has a smoother gradient landscape (no diverging gradient at $a=0$). While this might not be important in this simple example, in higher dimensions where the Jacobians of adjacent regions could be inconsistent (if the eigenvalues have different signs), it is useful to be able to cross regions where the Jacobian is singular without having to overcome an excessive gradient barrier.

\section{EXPERIMENTS}
\label{sec:experiments}

In this section, we demonstrate the practical capabilities of free-form flows (FFF). We mainly compare the performance against normalizing flows based on architectures which are invertible by construction. First, on an inverse problem benchmark, we show that using free-form architectures offers competitive performance to recent spline-based and ODE-based normalizing flows. This is achieved despite minimal tuning of hyperparameters, demonstrating that FFFs are easy to adapt to a new task. Second, on two molecule generation benchmarks, we demonstrate that specialized networks can now be used in a normalizing flow. In particular, we employ the equivariant graph neural networks $E(n)$-GNN \citep{satorras2021equivariant}. This $E(n)$-FFF outperforms ODE-based equivariant normalizing flows in terms of likelihood, and generates stable molecules significantly faster than a diffusion model.

\subsection{Simulation-Based Inference}

One popular application of generative models is in solving inverse problems. Here, the goal is to estimate hidden parameters from an observation. As inverse problems are typically ambiguous, a probability distribution represented by a generative model is a suitable solution. From a Bayesian perspective, this probability distribution is the posterior of the parameters given the observation. We learn this posterior via a conditional generative model.

In particular, we focus on simulation based inference (SBI, \cite{radev2022bayesflow, radev2021outbreakflow, bieringer2021measuring}), where we want to predict the parameters of a simulation. The training data is pairs of parameters and outputs generated from the simulation.

We train FFF models on the benchmark proposed in \citep{lueckmann2021benchmarking}, which is comprised of ten inverse problems of varying difficulty at three different simulation budgets (i.e.~training-set sizes) each. The models are evaluated via a classifier 2-sample test (C2ST) \citep{lopez2016revisiting, friedman2003multivariate}, where a classifier is trained to discern samples from the trained generative model and the true parameter posterior. The model performance is then reported as the classifier accuracy, where 0.5 demonstrates a distribution indistinguishable from the true posterior. We average this accuracy over ten different observations. In \cref{fig:sbi-c2st}, we report the C2ST of our model and compare it against the baseline based on neural spline flows \citep{durkan2019neural} and flow matching for SBI \citep{wildberger2023flow}.
Our method performs competitively, especially providing an improvement over existing methods in the regime of low simulation budgets. Regarding tuning of hyperparameters, we find that a simple fully-connected architecture with skip connections works across datasets with minor modifications to increase capacity for the larger datasets. We identify the reconstruction weight $\beta$ large enough such that training becomes stable. We give all dataset and more training details in \cref{app:sbi}.

\begin{figure*}
    \centering
    \includegraphics[width=\textwidth]{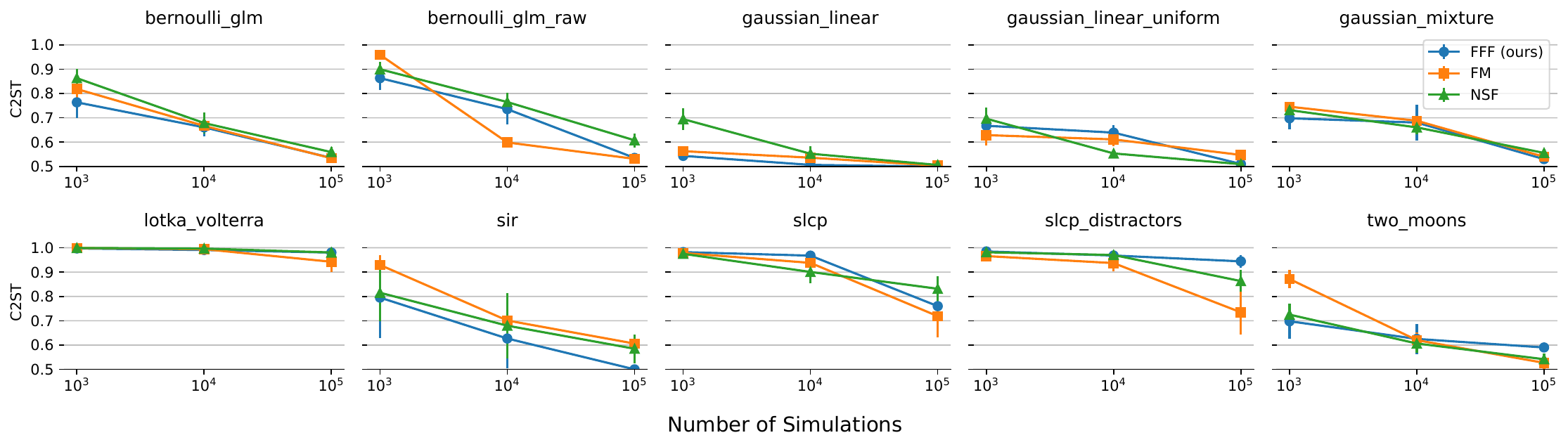}
    \caption{C2ST accuracy on the SBI benchmark datasets. We compare our method (FFF) against flow matching (FM) \cite{wildberger2023flow} and the neural spline flow (NSF) baseline in the benchmark dataset \cite{lueckmann2021benchmarking}. The accuracy is averaged over ten different observations, with error bars indicating the standard deviation. Our performance is comparable to the competitors across all datasets, with no model being universally better or worse.}
    \label{fig:sbi-c2st}
\end{figure*}

\subsection{Molecule Generation}

Free-form normalizing flows (FFF) do not make any assumptions about the underlying networks $f_\theta$ and $g_\phi$, except that they preserve dimension. We can leverage this flexibility for tasks where explicit constraints \textit{should} be built into the architecture, as opposed to constraints that originate from the need for tractable optimization (such as coupling blocks).

As a showcase, we apply FFF to molecule generation. Here, the task is to learn the joint distribution of a number of atoms $x_1, \dots, x_N \in \R^n$. Each prediction of the generative model should yield a physically valid position for each atom: $x = (x_1, \dots, x_N) \in \R^{N \times n}$.

The physical system of atoms in space have an important symmetry: if a molecule is shifted or rotated in space, its properties do not change. This means that a generative model for molecules should yield the same probability regardless of orientation and translation:
\begin{equation}
    p_\phi(Q x + t) = p_\phi(x).
    \label{eq:euclidean-invariance}
\end{equation}
Here, the rotation $Q \in \R^{n \times n}$ acts on $x$ by rotating or reflecting each atom $x_i \in \R^n$ about the origin, and $t \in \R^n$ applies the same translation to each atom.
Formally, $(Q, t)$ are realizations of the Euclidean group $E(n)$. The above \cref{eq:euclidean-invariance} means that the distribution $p_\phi(x)$ is invariant under the Euclidean group $E(n)$.

\cite{kohler2020equivariant,toth2020hamiltonian} showed that if the latent distribution $p(z)$ is \textit{invariant} under a group $G$, and a generative model $g_\phi(z)$ is \textit{equivariant} to $G$, then the resulting distribution is also invariant to $G$. Equivariance means that applying any group action to the input (e.g.~rotation and translation) and then applying $g_\phi$ should give the same result as first applying $g_\phi$ and then applying the group. For example, for the Euclidean group:
\begin{equation}
    Q g_\phi(z) + t = g_\phi(Qz + t).
    \label{eq:euclidean-equivariance}
\end{equation}

This implies that we can learn a distribution invariant to the Euclidean group by construction by making normalizing flows equivariant to the Euclidean group as in \cref{eq:euclidean-equivariance}. Previous work has demonstrated that this inductive bias is more effective than data augmentation, where random rotations and translations are applied to each data point at train time \citep{kohler2020equivariant,hoogeboom2022equivariant}.

We therefore choose an $E(n)$ equivariant network as the networks $f_\theta(x)$ and $g_\phi(z)$ in our FFF. We employ the $E(n)$-GNN proposed by \cite{satorras2021equivariant}. We call this model the $E(n)$-free-form flow ($E(n)$-FFF). We give the implementation details in \cref{app:molecule-generation}.

The $E(n)$-GNN has also been the backbone for previous normalizing flows on molecules. However, to the best of our knowledge, all realizations of such architectures have been based on neural ODEs, where the flow is parameterized as a differential equation $\frac{\mathrm{d} x}{\mathrm{d} t} = f_\theta(x(t), t)$. While training, one can avoid solving the ODE by using the rectified flow or flow matching objective \citep{liu2023learning, lipman2023flow, albergo2023building}. However, they still have the disadvantage that they require integrating the ODE for sampling. Our model, in contrast, only calls $f_\phi(z)$ once for sampling.

\begin{table}
    \centering
    \resizebox{\linewidth}{!}{
    \begin{tabular}{l|cc} 
        & \multirow{2}{*}{NLL ($\downarrow$)} & Sampling \\
        & & time ($\downarrow$) \\
        \hline
        & \multicolumn{2}{c}{DW4} \\ \hline
        $E(n)$-NF \citep{satorras2021equivariantflow} & 1.72 $\pm$ 0.01 &\textbf{0.024} ms \\  %
        OT-FM \citep{klein2023equivariant}   & \textbf{1.70} $\pm$ 0.02  &0.034 ms\\  %
        E-OT-FM \citep{klein2023equivariant} & \textbf{1.68} $\pm$ 0.01  &0.033 ms \\ %
        $E(n)$-FFF (ours) & \textbf{1.68} $\pm$ 0.01&\textbf{0.026} ms \\ %
        \hline
        
        & \multicolumn{2}{c}{LJ13} \\ \hline
        $E(n)$-NF & -16.28 $\pm$ 0.04 &0.27 ms  \\ %
        OT-FM & -16.54 $\pm$ 0.03 &0.77 ms  \\ %
        E-OT-FM & -16.70 $\pm$ 0.12 &0.72 ms \\ %
        $E(n)$-FFF (ours) & \textbf{-17.09} $\pm$ 0.16&\textbf{0.11} ms \\ %
        \hline
        
        & \multicolumn{2}{c}{LJ55} \\ \hline
        OT-FM & -88.45 $\pm$ 0.04   & 40 ms  \\ %
        E-OT-FM & \textbf{-89.27} $\pm$ 0.04 &40 ms  \\ %
        $E(n)$-FFF (ours) & -88.72 $\pm$ 0.16 &\textbf{2.1} ms \\ %
    \end{tabular}
    }
    \caption{%
    Equivariant free-form flows ($E(n)$-FFF) sample significantly faster than previous models, and achieve comparable or better negative log-likelihood (NLL, lower is better). More details in \cref{tab:bg-sampling-time}.}
    \label{tab:boltzmann-overview}
\end{table}

\paragraph{Boltzmann Generator} We test our $E(n)$-FFFs in learning a Boltzmann distribution:
\begin{equation}
    q(x) \propto e^{-\beta u(x)},
\end{equation}
where $u(x) \in \R$ is an energy function that takes the positions of atoms $x = (x_1, \dots, x_N)$ as an input.
A generative model $p_\phi(x)$ that approximates $q(x)$ can be used as a Boltzmann generator \citep{noe2019boltzmann}. The idea of the Boltzmann generator is that having access to $u(x)$ allows re-weighting samples from the generator after training even if $p_\phi(x)$ is different from $q(x)$. This is necessary in order to evaluate samples from $q(x)$ in a downstream task: Re-weighting samples allows computing expectation values $\E_{x \sim q(x)}[O(x)] = \E_{x \sim p_\phi(x)}[\frac{q(x)}{p_\phi(x)} O(x)]$ from samples of the generative model $p_\phi(x)$ if $p_\phi(x)$ and $q(x)$ have the same support. 

We evaluate the performance of free-form flows (FFF) as a Boltzmann generator on the benchmark tasks DW4, LJ13, and LJ55 \citep{kohler2020equivariant, klein2023equivariant}. Here, pairwise potentials $v(x_i, x_j)$ are summed as the total energy $u(x)$:
\begin{equation}
    u(x) = \sum_{i,j} v(x_i, x_j).
\end{equation}
DW4 uses a double-well potential $v_{\mathtt{DW}}$ and considers four particles in 2D. LJ13 and LJ55 both employ a Lennard-Jones potential $v_{\mathtt{LJ}}$ between 13 respectively 55 particles in 3D space (see \cref{app:bg} for details). We make use of the datasets presented by \cite{klein2023equivariant}, which obtained samples from $p(x)$ via MCMC.\footnote{Datasets available at: \url{https://osf.io/srqg7/?view_only=28deeba0845546fb96d1b2f355db0da5}}

In \cref{tab:boltzmann-overview}, we compare our model against (i) the equivariant ODE normalizing flow trained with maximum likelihood $E(n)$-NF \citep{satorras2021equivariantflow}, and (ii) two equivariant ODEs trained via optimal transport (equivariant) flow matching \cite{klein2023equivariant}. We find our model to have comparable or better negative log-likelihood than competitors. In addition, $E(n)$-FFFs sample significantly faster than competitors because our model needs to evaluate the learned network only once, as opposed to the multiple evaluations required to integrate an ODE.

\begin{table}
    \centering
    \newcommand{\ralign}[1]{\multicolumn{1}{r}{#1}}
    \resizebox{\linewidth}{!}{
    \begin{tabular}{l|cccc} %
         & \multirow{2}{*}{NLL ($\downarrow$)} & \multirow{2}{*}{Stable ($\uparrow$)} & %
         \multicolumn{2}{c}{Sampling time ($\downarrow$)} \\
         &&& Raw & Stable \\ \hline 
        $E(3)$-NF & -59.7 & 4.9 \% & %
        \ralign{13.9 ms} & \ralign{309.5 ms} \\
        $E(3)$-DM & \textbf{-110.7} & \textbf{82.0} \% & 1580.8 ms& %
        \ralign{1970.6 ms}\\ 
        $E(3)$-FFF & -76.2& 8.7 \%& \ralign{\textbf{0.6} ms}& %
        \ralign{\textbf{8.1} ms}\\ \hline
        Data & - & 95.2 \% & %
        - & - \\ %
    \end{tabular}
    }
    \caption{$E(3)$-FFF (ours) trained on QM9 generates a stable molecule faster than previous models because a sample is obtained via a single function evaluation. $E(3)$-DM is the $E(3)$-diffusion model \citep{hoogeboom2022equivariant}, $E(3)$-NF the $E(3)$-normalizing flow \citep{satorras2021equivariantflow}. The latter is also trained explicitly using maximum likelihood, yet outperformed by $E(3)$-FFF in terms of negative log-likelihood (NLL) and what ratio of generated molecules is stable.}
    \label{tab:molecule-qm9}
\end{table}

\paragraph{QM9 Molecules} As a second molecule generation benchmark, we test the performance of $E(3)$-FFF in generating novel molecules. We therefore train on the QM9 dataset \citep{ruddigkeit2012, ramakrishnan2014}, which contains molecules of varying atom count, with the largest molecules counting $29$ atoms. The goal of the generative model is not only to predict the positions of the atoms in each molecule $x = (x_1, \dots, x_N) \in \R^3$, but also each atom's properties $h_i$ (atom type (categorical), and atom charge (ordinal)).

We again employ the $E(3)$-GNN \citep{satorras2021equivariant}. The part of the network that acts on coordinates $x_i \in \R^3$ is equivariant to rotations, reflections and translations (Euclidean group $E(3)$). The network leaves the atom properties $h$ invariant under these operations.

We show samples from our model in \cref{fig:molecule-qm9}. Because free-form flows only need one network evaluation to sample, they generate two orders of magnitude more stable molecules than the $E(3)$-diffusion model \citep{hoogeboom2022equivariant} and one order of magnitude more than the $E(3)$-normalizing flow \citep{satorras2021equivariantflow} in a fixed time window, see \cref{tab:molecule-qm9}. This includes the time to generate unstable samples, which are discarded. A molecule is called stable if each atom has the correct number of bonds, where bonds are determined from inter-atomic distances.
$E(3)$-FFF also outperforms $E(3)$-NF trained with maximum likelihood both in terms of likelihood and in how many of the sampled molecules are stable. See \cref{app:qm9} for implementation details.

\section{CONCLUSION}

In this work, we present free-form flows (FFF), a new paradigm for normalizing flows that enables training arbitrary dimension-preserving neural networks with maximum likelihood. Invertibility is achieved by a reconstruction loss and the likelihood is maximized by an efficient surrogate. Previously, designing normalizing flows was constrained by the need for analytical invertibility. Free-form flows allow practitioners to focus on the data and suitable inductive biases instead.

We show that free-form flows are an exact relaxation of maximum likelihood training, converging to the same solutions provided that the reconstruction loss is minimal. We provide an interpretation of FFF training as the minimization of a lower bound on the KL divergence between noisy versions of the data and the generative distribution. Furthermore this bound is tight if $f_\theta$ and $g_\phi$ are true inverses.

In practice, free-form flows perform on par or better than previous normalizing flows, exhibit fast sampling by only requiring a single function evaluation, and are easy to tune. We provide a practical guide for adapting them to new problems in \cref{app:practical-guide}.

\acknowledgments{This work is supported by Deutsche Forschungsgemeinschaft (DFG, German Research Foundation) under Germany's Excellence Strategy EXC-2181/1 - 390900948 (the Heidelberg STRUCTURES Cluster of Excellence). It is also supported by the Vector Stiftung in the project TRINN (P2019-0092). AR acknowledges funding from the Carl-Zeiss-Stiftung. LZ ackknowledges support by the German Federal Ministery of Education and Research (BMBF) (project EMUNE/031L0293A).
The authors acknowledge support by the state of Baden-Württemberg through bwHPC and the German Research Foundation (DFG) through grant INST 35/1597-1 FUGG.}

\bibliography{references}

\section*{Checklist}

 \begin{enumerate}

 \item For all models and algorithms presented, check if you include:
 \begin{enumerate}
   \item A clear description of the mathematical setting, assumptions, algorithm, and/or model. Yes
   \item An analysis of the properties and complexity (time, space, sample size) of any algorithm. Yes
   \item (Optional) Anonymized source code, with specification of all dependencies, including external libraries. Not applicable
 \end{enumerate}

 \item For any theoretical claim, check if you include:
 \begin{enumerate}
   \item Statements of the full set of assumptions of all theoretical results. Yes
   \item Complete proofs of all theoretical results. Yes
   \item Clear explanations of any assumptions. Yes
 \end{enumerate}

 \item For all figures and tables that present empirical results, check if you include:
 \begin{enumerate}
    \item The code, data, and instructions needed to reproduce the main experimental results (either in the supplemental material or as a URL). Yes
    \item All the training details (e.g., data splits, hyperparameters, how they were chosen). Yes
    \item A clear definition of the specific measure or statistics and error bars (e.g., with respect to the random seed after running experiments multiple times). Yes
    \item A description of the computing infrastructure used. (e.g., type of GPUs, internal cluster, or cloud provider). Yes
 \end{enumerate}

 \item If you are using existing assets (e.g., code, data, models) or curating/releasing new assets, check if you include:
 \begin{enumerate}
   \item Citations of the creator If your work uses existing assets. Yes
   \item The license information of the assets, if applicable. Not applicable
   \item New assets either in the supplemental material or as a URL, if applicable. Not applicable
   \item Information about consent from data providers/curators. Not applicable
   \item Discussion of sensible content if applicable, e.g., personally identifiable information or offensive content. Not applicable
 \end{enumerate}

 \item If you used crowdsourcing or conducted research with human subjects, check if you include:
 \begin{enumerate}
   \item The full text of instructions given to participants and screenshots. Not applicable
   \item Descriptions of potential participant risks, with links to Institutional Review Board (IRB) approvals if applicable. Not applicable
   \item The estimated hourly wage paid to participants and the total amount spent on participant compensation. Not applicable
 \end{enumerate}

 \end{enumerate}

\onecolumn
\aistatstitle{Supplementary Materials}

\appendix

\section*{OVERVIEW}

The appendix is structured into three parts:
\begin{itemize}
    \item \Cref{app:theory}: A restatement and proof of all theoretical claims in the main text, along with some additional results.
    \begin{itemize}
        \item \Cref{app:gradient-via-trace}: The gradient of the log-determiant can be written as a trace.
        \item \Cref{app:loss-derivation}: A derivation of the loss as a lower bound on a KL divergence.
        \item \Cref{app:error-bound}: A bound on the difference between the true gradient of the log-determinant and the estimator used in this work.
        \item \Cref{app:critical-points}: Properties of the critical points of the loss.
        \item \Cref{app:global-invertibility}: Exploration of behavior of the loss in the low $\beta$ regime, where the solution may not be globally invertible.
    \end{itemize}
    \item \Cref{app:practical-guide}: Practical tips on how to train free-form flows and adapt them to new problems.
    \begin{itemize}
        \item \Cref{app:model-setup}: Tips on how to set up and initialize the model.
        \item \Cref{app:training}: Code for computing the loss function.
        \item \Cref{app:likelihood}: Details on how to estimate likelihoods.
        \item \Cref{app:reconstruction-weight}: Tips on how to tune $\beta$.
    \end{itemize}
    \item \Cref{app:experiments}: Details necessary to reproduce all experimental results in the main text.
    \begin{itemize}
        \item \Cref{app:sbi}: Simulation-based inference.
        \item \Cref{app:molecule-generation}: Molecule generation.
    \end{itemize}
\end{itemize}

\vfill
\newpage

\section{THEORETICAL CLAIMS}
\label{app:theory}

This section contains restatements and proofs of all theoretical claims in the main text.

\subsection{Gradient via Trace}
\label{app:gradient-via-trace}

\begin{theorem}
\label{appthm:gradient-via-trace}
    Let $f_\theta: \R^D \to \R^D$ be a $C^1$ invertible function parameterized by $\theta$. Then, for all $x \in \R^D$:
    \begin{equation}
        \nabla_{\theta_i} \log|J_\theta(x)| = \tr\!\left(
            (\nabla_{\theta_i} J_\theta(x)) 
            (J_\theta(x))^{-1}
        \right).
    \end{equation}
\end{theorem}

\begin{proof}
Jacobi's formula states that, for a matrix $A(t)$ parameterized by $t$, the derivative of the determinant is
\begin{equation}
    \frac{\mathrm{d}}{\mathrm{d}t} |A(t)| = |A(t)| \tr \left( A(t)^{-1} \frac{\mathrm{d}A(t)}{\mathrm{d}t} \right)
\end{equation}
and hence
\begin{align}
    \frac{\mathrm{d}}{\mathrm{d}t} \log |A(t)| 
        &= |A(t)|^{-1} \frac{\mathrm{d}}{\mathrm{d}t} |A(t)| \\
        &= \tr \left( A(t)^{-1} \frac{\mathrm{d}A(t)}{\mathrm{d}t} \right) \\
        &= \tr \left( \frac{\mathrm{d}A(t)}{\mathrm{d}t} A(t)^{-1} \right)
\end{align}
using the cyclic property of the trace in the last step.
Applying this formula, with $A = J_\theta(x)$ and $t = \theta_i$ gives the result.
\end{proof}

\subsection{Loss Derivation}
\label{app:loss-derivation}

Here we derive the loss function via an upper bound on a Kullback-Leibler (KL) divergence. Before doing so, let us establish some notation and motivation.

Our generative model is as follows: 
\begin{align}
    p(z) &= \N(z; 0, I) \\
    p_\phi(x \mid z) &= \delta(x - g_\phi(z)) 
\end{align}
meaning that to generate data we sample from a standard normal latent distribution and pass the sample through the generator network $g_\phi$.
The corresponding inference model is:
\begin{align}
    q(x) &= \text{data distribution} \\
    q_\theta(z \mid x) &= \delta(z - f_\theta(x))
\end{align}

Our goal is to minimize the KL divergence
\begin{align}
    \KL(q(x) \parallel p_\phi(x)) 
        &= \E_{q(x)} \left[ \log \frac{q(x)}{p_\phi(x)} \right] \\
        &= \E_{q(x)} \left[ - \log \int p_\phi(x, z) \mathrm{d}z \right] - h(q(x))
\end{align}
where $h$ denotes the differential entropy.
Unfortunately this divergence is intractable, due to the integral over $z$ (though it would be tractable if $g_\phi^{-1}$ and $\log|J_{g_\phi}(z)|$ are tractable due to the change of variables formula -- in this case the model would be a typical normalizing flow).
The variational autoencoder (VAE) is a latent variable model which solves this problem by minimizing
\begin{align}
    \KL(q_\theta(x, z) \parallel p_\phi(x, z)) 
        &= E_{q_\theta(x, z)} \left[ \log \frac{q_\theta(x, z)}{p_\phi(x, z)} \right] \\
        &= E_{q_\theta(x, z)} \left[ \log \frac{q(x)}{p_\phi(x)} + \log \frac{q_\theta(z \mid x)}{p_\phi(z \mid x)}
        \right] \\
        &= \KL(q(x) \parallel p_\phi(x)) + E_{q(x)} \left[ \KL(q_\theta(z \mid x) \parallel p_\phi(z \mid x)) \right] \\
        &\geq \KL(q(x) \parallel p_\phi(x))
\end{align}
The inequality comes from the fact that KL divergences are always non-negative.
Unfortunately this KL divergence is not well-defined due to the delta distributions, which make the joint distributions over $x$ and $z$ degenerate. Unless the support of $q_\theta(x, z)$ and $p_\phi(x, z)$ exactly overlap, which is very unlikely for arbitrary $f_\theta$ and $g_\phi$, the divergence will be infinite. The solution is to introduce an auxiliary variable $\tilde x$ which is the data with some added Gaussian noise:
\begin{equation}
    p(\tilde x \mid x) = q(\tilde x \mid x) = \N(\tilde x; x, \sigma^2 I)
\end{equation}
The generative model over $z$ and $\tilde x$ is therefore
\begin{align}
    p(z) &= \N(z; 0, I) \\
    p_\phi(\tilde x \mid z) &= \N(\tilde x; g_\phi(z), \sigma^2 I)
\end{align}
and the inference model is
\begin{align}
    q(\tilde x) &= \int q(x) q(\tilde x \mid x) \mathrm{d}x \\
    q(\tilde x \mid x) &= \N(\tilde x; x, \sigma^2 I) \\
    q_\theta(z \mid \tilde x) &= \frac{\int q(x) q(\tilde x \mid x) q_\theta(z \mid x) \mathrm{d}x}{\int q(x) q(\tilde x \mid x) \mathrm{d}x}
\end{align}

Now the relationship between $\tilde x$ and $z$ is stochastic and we can safely minimize the KL divergence which will always take on finite values:
\begin{equation}
    \KL(q_\theta(\tilde x, z) \parallel p_\phi(\tilde x, z)) \geq \KL(q(\tilde x) \parallel p_\phi(\tilde x))
\end{equation}
The KL divergence between the noised variables is known as a spread KL divergence $\tilde{\mathcal{D}}_\text{KL}$ \citep{zhang2020spread}:
\begin{equation}
    \tilde{\mathcal{D}}_\text{KL} (q(x) \parallel p_\phi(x)) = \KL (q(\tilde x) \parallel p_\phi(\tilde x))
\end{equation}
For convenience, here is the definition of $\loss^{f^{-1}}$:
\begin{equation} \label{eq:loss-f-inv}
    \loss^{f^{-1}} = \E_{q(x)p(v)} \left[ -\log p(Z = f_\theta(x)) - v^T J_\theta \texttt{SG}( J_\theta^{-1} v)   + \beta \lVert x - g_\phi(f_\theta(x)) \rVert^2 \right]
\end{equation}
which has the same gradients (with respect to model parameters) as
\begin{equation}
    \E_{q(x)} \left[ -\log p(Z = f_\theta(x)) - \log|J_\theta| + \beta \lVert x - g(f(x)) \rVert^2 \right]
\end{equation}
Now we restate the theorem from the main text:

\begin{theorem}
\label{appthm:loss-derivation}
Let $f_\theta$ and $g_\phi$ be $C^1$ and let $f_\theta$ be globally invertible. Define the spread KL divergence $\tilde{\mathcal{D}}_\text{KL}$ as the KL divergence between distributions convolved with isotropic Gaussian noise of variance $\sigma^2$. Let $\beta = 1/2\sigma^2$. Then there exists a function $D$ of $\theta$ and $\phi$ such that
\begin{equation}
    \nabla_\theta \loss^{f^{-1}} = \nabla_\theta D \quad \text{and} \quad \nabla_\phi \loss^{f^{-1}} = \nabla_\phi D
\end{equation}
and
\begin{equation}
     D \geq \tilde{\mathcal{D}}_\text{KL}( q(x) \parallel p_\phi(x) )
\end{equation}
As a result, minimizing $\loss^{f^{-1}}$ is equivalent to minimizing an upper bound on $\tilde{\mathcal{D}}_\text{KL}( q(x) \parallel p_\phi(x) )$. 
\end{theorem}

\begin{proof}
Let 
\begin{equation}
    D = \KL(q_\theta(\tilde x, z) \parallel p_\phi(\tilde x, z))
\end{equation}
We will use the identity \citep{papoulis2002probability}
\begin{equation}
    h(Z) = h(X) + \E [ \log |J_f(X)| ]
\end{equation}
where $h$ is the differential entropy, and the random variables are related by $Z = f(X)$ where $f$ is invertible. 
As a result,
\begin{equation} \label{eq:entropy-q-z-tilde-x}
    h(q(z \mid \tilde x)) = h(q(x \mid \tilde x)) + \E_{q(x \mid \tilde x)} [ \log |J_f(x)| ] = \E_{q(x)} [ \log |J_f(x)| ] + \text{const.}
\end{equation}

In the following, drop $\theta$ and $\phi$ subscripts for convenience. The unspecified extra terms are constant with respect to network parameters. Let $\epsilon$ be a standard normal variable. Then:
\begin{align}
    D   &= \E_{q(\tilde x, z)} \left[ \log q(\tilde x) + \log q(z \mid \tilde x) - \log p(z) - \log p(\tilde x \mid z) \right] \label{eq:D1} \\
        &= \E_{q(\tilde x)} \left[ - h(q(z \mid \tilde x)) \right] + \E_{q(\tilde x, z)} \left[ - \log p(z) - \log p(\tilde x \mid z) \right] + \text{const.} \label{eq:D2} \\
        &= \E_{q(x)} \left[ - \log |J_f(x)| \right] + \E_{q(\tilde x, z)} \left[ - \log p(z) - \log p(\tilde x \mid z) \right] + \text{const.} \label{eq:D3} \\
        &= \E_{q(x)q(\epsilon)} \left[ - \log |J_f(x)| - \log p(Z = f(x)) - \log p(\tilde X = x + \sigma\epsilon \mid Z = f(x)) \right] + \text{const.} \label{eq:D4} \\
        &= \E_{q(x)q(\epsilon)} \left[ -\log |J_f(x)| - \log p(Z = f(x)) + \frac{1}{2\sigma^2} \lVert x + \sigma\epsilon - g(f(x)) \rVert^2 \right] + \text{const.} \label{eq:D5} \\
        &= \E_{q(x)q(\epsilon)} \left[ -\log |J_f(x)| - \log p(Z = f(x)) + \frac{1}{2\sigma^2} \lVert x - g(f(x)) \rVert^2 + \frac{1}{\sigma}\epsilon^\top (x - g(f(x))) \right] + \text{const.} \label{eq:D6} \\
        &= \E_{q(x)} \left[ -\log |J_f(x)| - \log p(Z = f(x)) + \beta \lVert x - g(f(x)) \rVert^2 \right] + \text{const.} \label{eq:D7}
\end{align}
Where the following steps were taken:
\begin{itemize}
    \item Regard $\E_{q(\tilde x)}[\log q(\tilde x)] = -h(q(\tilde x))$ as a constant (\cref{eq:D2})
    \item Substitute in \cref{eq:entropy-q-z-tilde-x} and regard $h(q(x \mid \tilde x))$ as constant (\cref{eq:D3})
    \item Make a change of variables from $\tilde x, z$ to $x, \epsilon$ with $\tilde x = x + \sigma\epsilon$ and $z = f(x)$ (\cref{eq:D4})
    \item Substitute the log-likelihood of the Gaussian $p(\tilde x \mid z)$, discard constant terms (\cref{eq:D5})
    \item Expand the final quadratic term, discard the constant term in $\lVert \epsilon \rVert^2$ (\cref{eq:D6})
    \item Evaluate the expectation over $\epsilon$, noting that $\epsilon$ is independent of $x$ and $\E[\epsilon] = 0$. Substitute $\beta = 1/2\sigma^2$ (\cref{eq:D7})
    \item  Recognize that the final expression has the same gradient as $\loss^{f^{-1}}$ from \cref{eq:loss-f-inv}
\end{itemize}

Since the extra terms are constant with respect to $\theta$ and $\phi$ we have
\begin{equation}
    \nabla_\theta \loss^{f^{-1}} = \nabla_\theta D \quad \text{and} \quad \nabla_\phi \loss^{f^{-1}} = \nabla_\phi D
\end{equation}
and
\begin{equation}
     D \geq \KL( q(\tilde x) \parallel p_\phi(\tilde x) ) = \tilde{\mathcal{D}}_\text{KL} (q(x) \parallel p_\phi(x))
\end{equation}
was already established. As a result the gradients of $\loss^{f^{-1}}$ are an unbiased estimate of the gradients of $D$ and minimizing $\loss^{f^{-1}}$ under stochastic gradient descent will converge to the same solutions as when minimizing $D$.
\end{proof}

We can also demonstrate the related bound:
\begin{equation}
    D' \geq \KL(q(x) \parallel p_\theta(x))
\end{equation}
where $D' = \loss^{f^{-1}} + \text{const.}$ (with a different constant to $D$). This is easy to see, since $\loss^{f^{-1}} = \loss^{f^{-1}}_\text{ML} + \beta \loss^{f^{-1}}_\text{R}$ and $\loss^{f^{-1}}_\text{ML} = \KL(q(x) \parallel p_\theta(x)) + \text{const.}$ and $\loss^{f^{-1}}_\text{R} \geq 0$.

\subsection{Error Bound}
\label{app:error-bound}

\begin{theorem}
\label{appthm:error-bound}
Let $f_\theta$ and $g_\phi$ be $C^1$, let $J_\theta$ be the Jacobian of $f_\theta$ at $x$ and let $J_\phi$ be the Jacobian of $g_\phi$ at $f_\theta(x)$. Suppose that $f_\theta$ is locally invertible at $x$, meaning $J_\theta(x)$ is an invertible matrix. Let $\lVert \cdot \rVert_F$ be the Frobenius norm of a matrix. Then the absolute difference between $\nabla_{\theta_i} \log |J_\theta(x)|$ and the trace-based approximation is bounded:
\begin{equation}
    \left| \tr((\nabla_{\theta_i} J_\theta) J_\phi ) - \nabla_{\theta_i} \log |J_\theta| \right| \leq \lVert J_\theta^{-1} \nabla_{\theta_i} J_\theta \rVert_F \lVert J_\phi J_\theta - I \rVert_F
\end{equation}
\end{theorem}

\begin{proof}
The Cauchy-Schwarz inequality states that, for an inner product $\langle \cdot, \cdot \rangle$
\begin{equation}
    |\langle u, v \rangle|^2 \leq \langle u, u \rangle \langle v, v \rangle
\end{equation}
The trace forms the so-called Frobenius inner product over matrices with $\langle A, B \rangle_F = \tr(A^T B)$. Applying the inequality gives
\begin{align}
    |\tr(A^T B)|^2 
        &\leq \tr(A^T A) \tr(B^T B) \\
        &= \lVert A \rVert_F^2 \lVert B \rVert_F^2
\end{align}
with $\lVert A \rVert_F = \sqrt{\tr(A^T A)}$ the Frobenius norm of $A$.

Recall from \cref{appthm:gradient-via-trace} that
\begin{equation}
    \nabla_{\theta_i} \log |J_\theta| = \tr((\nabla_{\theta_i} J_\theta) J_\theta^{-1})
\end{equation}
Therefore
\begin{align}
    \left| \tr((\nabla_{\theta_i} J_\theta) J_\phi) - \nabla_{\theta_i} \log |J_\theta| \right|
        &= \left| \tr((\nabla_{\theta_i} J_\theta) J_\phi) - \tr((\nabla_{\theta_i} J_\theta) J_\theta^{-1}) \right| \\
        &= \left| \tr((\nabla_{\theta_i} J_\theta)(J_\phi - J_\theta^{-1})) \right| \\
        &= \left| \tr((\nabla_{\theta_i} J_\theta) J_\theta^{-1} (J_\theta J_\phi  - I) ) \right| \\
        &\leq \lVert (\nabla_{\theta_i} J_\theta) J_\theta^{-1} \rVert_F \lVert J_\theta J_\phi - I \rVert_F
\end{align}
where the last line is application of the Cauchy-Schwarz inequality and the cyclicity of the trace.
\end{proof}

\begin{theorem}
Suppose the conditions of \cref{appthm:error-bound} hold but extend local invertibility of $f_\theta$ to invertibility wherever $q(x)$ has support. Then the difference in gradients between $\loss^g$ and $\loss^{f^{-1}}$ is bounded:
\begin{equation}
    \left| \nabla_{\theta_i} \loss^g - \nabla_{\theta_i} \loss^{f^{-1}} \right| \leq \E_{q(x)} \left[ \lVert (\nabla_{\theta_i} J_\theta) J_\theta^{-1} \rVert_F^2 \right]^{\frac{1}{2}} \E_{q(x)} \left[ \lVert J_\theta J_\phi - I \rVert_F^2 \right]^{\frac{1}{2}}
\end{equation}
\end{theorem}

\begin{proof}
In addition to the Cauchy-Schwarz inequality used in the proof to \cref{appthm:error-bound}, we will also require Jensen's inequality for a convex function $\alpha: \R \rightarrow \R$
\begin{equation}
    \alpha(\E_{q(x)}[x]) \leq \E_{q(x)}[\alpha(x)]
\end{equation}
and Hölder's inequality (with $p=q=2$) for random variables $X$ and $Y$
\begin{equation}
    \E[|XY|] \leq \E[|X|^2]^{\frac{1}{2}} \E[|Y|^2]^{\frac{1}{2}}
\end{equation}

The only difference between $\loss^g$ and $\loss^{f^{-1}}$ is in the estimation of the gradient of the log-determinant. We use this fact, along with the inequalities, which we apply in the Jensen, Cauchy-Schwarz, Hölder order:
\begin{align}
    \left| \nabla_{\theta_i} \loss^g - \nabla_{\theta_i} \loss^{f^{-1}} \right|
        &= \left| \E_{q(x)} \left[ \tr((\nabla_{\theta_i} J_\theta) J_\phi) \right] - \E_{q(x)} \left[ \nabla_{\theta_i} \log |J_\theta| \right] \right| \\
        &= \left| \E_{q(x)} \left[ \tr((\nabla_{\theta_i} J_\theta) J_\theta^{-1}(J_\theta J_\phi - I)) \right] \right| \\
        &\leq \E_{q(x)} \left[ \left| \tr((\nabla_{\theta_i} J_\theta) J_\theta^{-1}(J_\theta J_\phi - I)) \right| \right] \\
        &\leq \E_{q(x)} \left[ \lVert (\nabla_{\theta_i} J_\theta) J_\theta^{-1} \rVert_F \lVert J_\theta J_\phi - I \rVert_F \right] \\
        &\leq \E_{q(x)} \left[ \lVert (\nabla_{\theta_i} J_\theta) J_\theta^{-1} \rVert_F^2 \right]^{\frac{1}{2}} \E_{q(x)} \left[ \lVert J_\theta J_\phi - I \rVert_F^2 \right]^{\frac{1}{2}}
\end{align}
\end{proof}

\subsection{Critical Points}
\label{app:critical-points}

\begin{theorem}
\label{appthm:critical-points}
Let $f_\theta$ and $g_\phi$ be $C^1$ and let $f_\theta$ be globally invertible. Suppose $q(x)$ is finite and has support everywhere. Then the critical points of $\loss^{f^{-1}}$ (for any $\beta > 0$) are such that 
\begin{enumerate}
    \item $g_\phi(z) = f_\theta^{-1}(z)$ for all $z$, and
    \item $p_\theta(x) = q(x)$ for all $x$, and
    \item All critical points are global minima
\end{enumerate}

Furthermore, every minimum of $\loss^{f^{-1}}$ is a critical point of $\loss^g$. If the reconstruction loss is minimal, $\loss^g$ has no additional critical points.
\end{theorem}

\begin{proof}
In the following we will use Einstein notation, meaning that repeated indices are summed over. For example, $a_i b_i$ is shorthand for $\sum_i a_i b_i$. We will drop $\theta$ and $\phi$ subscripts to avoid clutter. We will also use primes to denote derivatives, for example: $f'(x) = J_f(x)$. In addition, gradients with respect to parameters should be understood as representing the gradient of a single parameter at a time, so $\nabla_\theta \loss$ is shorthand for $(\nabla_{\theta_1} \loss, \dots)$.

We will use the calculus of variations to find the critical points on a functional level. For a primer on calculus of variations, please see \cite{weinstock1974calculus}. Our loss is of the form
\begin{equation}
    \loss^{f^{-1}} = \int \lambda(x, f, f', g) \mathrm{d}x
\end{equation}
with 
\begin{equation}
    \lambda(x, f, f', g) = q(x) \left( \frac{1}{2} \lVert f(x) \rVert^2 - \log |f'(x)| + \beta \lVert g(f(x)) - x \rVert^2 \right)
\end{equation}
By the Euler-Lagrange equations, critical points satisfy
\begin{equation}
    \frac{\partial \lambda}{\partial g_i} = 0
\end{equation}
for all $i$ and 
\begin{equation}
    \frac{\partial \lambda}{\partial f_i} - \frac{\partial}{\partial x_j}\left(\frac{\partial \lambda}{\partial f'_{ij}} \right) = 0
\end{equation}
for all $i$.

Taking the derivative with respect to $g$:
\begin{equation}
    \frac{\partial \lambda}{\partial g_i} = q(x) \cdot 2\beta (g(f(x)) - x)_i = 0
\end{equation}
and hence $g(f(x)) - x = 0$ for all $x$ (since $q(x) > 0$).
By a change of variables with $z = f(x)$, this means $g = f^{-1}$. Therefore we have proven statement 1.

Now differentiating with respect to $f$ and substituting $g = f^{-1}$:
\begin{align}
    \frac{\partial \lambda}{\partial f_i} 
        &= q(x) \left( f_i(x) + 2\beta (g(f(x)) - x)_j g'_{ji}(f(x)) \right) \\
        &= q(x) f_i(x)
\end{align}
and with respect to $f'$:
\begin{align}
    \frac{\partial \lambda}{\partial f'_{ij}} 
        &= -q(x) (f'(x)^{-1})_{lk} \frac{\partial f'_{kl}}{\partial f'_{ij}} \\
        &= -q(x) (f'(x)^{-1})_{ji}
\end{align}
meaning
\begin{align}
    \frac{\partial}{\partial x_j}\left(\frac{\partial \lambda}{\partial f'_{ij}} \right) 
        &= - \frac{\partial}{\partial x_j} q(x) (f'(x)^{-1})_{ji} - q(x) \frac{\partial}{\partial x_j} (f'(x)^{-1})_{ji} \\
        &= - q(x) \left( \frac{\partial}{\partial x_j} \log q(x) (f'(x)^{-1})_{ji} + \frac{\partial}{\partial x_j} (f'(x)^{-1})_{ji} \right)
\end{align}
Putting it together means
\begin{equation}
    q(x) \left( f_i(x) + \frac{\partial}{\partial x_j} \log q(x) (f'(x)^{-1})_{ji} + \frac{\partial}{\partial x_j} (f'(x)^{-1})_{ji} \right) = 0
\end{equation}

By dividing by $q(x)$ and multiplying by $f'_{ik}(x)$, we have
\begin{equation}
\label{eq:deriv-log-q-wrt-x}
    \frac{\partial}{\partial x_k} \log q(x) = -f_i(x) f'_{ik}(x) - \frac{\partial}{\partial x_j} (f'(x)^{-1})_{ji} f'_{ik}(x)
\end{equation}
Furthermore, since $f(x)$ is invertible:
\begin{equation}
    \frac{\partial}{\partial x_j} \left(f'(x)^{-1} f'(x) \right)_{jk} = \frac{\partial}{\partial x_j} \delta_{jk} = 0
\end{equation}
Then using the product rule:
\begin{equation}
\label{eq:product-f'-f'-inv}
    \frac{\partial}{\partial x_j}(f'(x)^{-1})_{ji} f'_{ik}(x) + (f'(x)^{-1})_{ji} f''_{ikj}(x) = 0
\end{equation}
In addition, 
\begin{equation}
\label{eq:jacobi-wrt-x}
    \frac{\partial}{\partial x_k} \log |f'(x)| = (f'(x)^{-1})_{ji} f''_{ijk}(x)
\end{equation}
from Jacobi's formula, and since Hessians are symmetric in their derivatives, we can put together \cref{eq:product-f'-f'-inv} and \cref{eq:jacobi-wrt-x} to form
\begin{equation}
    \frac{\partial}{\partial x_j} (f'(x)^{-1})_{ji} f'_{ik}(x) = -\frac{\partial}{\partial x_k} \log |f'(x)|
\end{equation}
Substituting into \cref{eq:deriv-log-q-wrt-x} and integrating, we find
\begin{equation}
    \log q(x) = - \frac{1}{2} \lVert f(x) \rVert^2 + \log |f'(x)| + \text{const.}
\end{equation}
The RHS is $\log p_\theta(x)$ by the change of variables formula, and hence $p_\theta(x) = q(x)$ for all $x$. This proves statement 2.

Now we will show that all critical points are global minima.

The negative log-likelihood part of $\loss$ is bounded below by $h(q(x))$ and the reconstruction part is bounded below by zero. Hence if all critical points achieve a loss of $h(q(x))$ they are all global minima.

Since $g = f^{-1}$ for all critical points, the reconstruction loss is zero.

Since $p_\theta(x) = q(x)$ for all critical points, the negative log-likelihood loss is $h(q(x))$:
\begin{equation}
    \E_{q(x)}[-\log p_\theta(x)] = \E_{q(x)}[-\log q(x)] = h(q(x))
\end{equation}
This proves statement 3.

It now remains to show that every minimum of $\loss^{f^{-1}}$ is a critical point of $\loss^g$.

$\loss^g$ is of the form:
\begin{equation}
    \loss^g = \int \tilde \lambda(x, f, f', g) \mathrm{d}x
\end{equation}
with
\begin{equation}
    \tilde \lambda(x, f, f', g) = q(x) \left( \frac{1}{2} \lVert f(x) \rVert^2 - \tr(f'(x) \texttt{SG}(g'(f(x)))) + \beta \lVert g(f(x)) - x \rVert^2 \right)
\end{equation}
\begin{equation}
    \frac{\partial \tilde \lambda}{\partial g_i} = q(x) \cdot 2\beta (g(f(x)) - x)_i
\end{equation}
as before, and is zero with the substitution $g = f^{-1}$.

\begin{equation}
    \frac{\partial \tilde \lambda}{\partial f_i} = q(x) f_i(x)
\end{equation}
as before and
\begin{align}
    \frac{\partial \tilde \lambda}{\partial f'_{ij}} 
        &= -q(x) g'_{lk}(f(x)) \frac{\partial f'_{kl}}{\partial f'_{ij}} \\
        &= -q(x) g'_{ji}(f(x)) \\
        &= -q(x) (f'(x)^{-1})_{ji}
\end{align}
with the substitution $g'(f(x)) = f'(x)^{-1}$. Since this is the same expression as before we must have
\begin{equation}
    \frac{\partial \tilde \lambda}{\partial f_i} - \frac{\partial}{\partial x_j}\left(\frac{\partial \tilde \lambda}{\partial f'_{ij}} \right) = 0
\end{equation}
meaning that $f$ and $g$ are critical with respect to $\tilde \loss$. This shows that the critical points (and hence minima) of $\loss^{f^{-1}}$ are critical points of $\loss^g$. 

In the case where $f$ is not required to be globally invertible, $\loss^g$ may have additional critical points when $f$ is in fact not invertible (see \cref{fig:critical-points} in \cref{sec:critical-points} for an example). However, if the reconstruction loss is minimal, therefore zero, $f$ will be invertible and the above arguments hold. If this is the case, there are no additional critical points of $\loss^g$.

\end{proof}

\subsection{Ensuring Global Invertibility}
\label{app:global-invertibility}

Free-form flows use arbitrary neural networks $f_\theta$ and $g_\phi$. Since we rely on the approximation $g_\phi \approx f_\theta^{-1}$, it is crucial that the reconstruction loss is as small as possible. We achieve this in practice by setting $\beta$ large enough. In this section, we give the reasoning for this choice.

In particular, we show that when $\beta$ is too small and the data is made up of multiple disconnected components, there are solutions to $\loss^{f^{-1}}$ that are not globally invertible, even if $f_\theta$ is restricted to be locally invertible. We illustrate some of these solutions for a two-component Gaussian mixture in \cref{fig:gmm-fn-of-beta}. We approximate the density as zero more than 5 standard deviations away from each mean. When $\beta$ is extremely low the model gives up on reconstruction and just tries to transform each component to the latent distribution individually.

\begin{figure}
    \centering
    \includegraphics[width=1\linewidth]{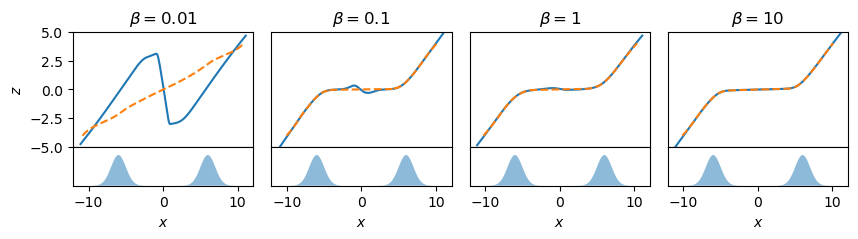}
    \caption{Solutions to $\loss^{f^{-1}}$ for various $\beta$. The data is the two-component Gaussian mixture shown in the lower panels. Solid blue lines show $f_\theta$ and dashed orange lines show $g_\phi$. Note that $f_\theta$ is not invertible between the mixtures when $\beta$ is small.}
    \label{fig:gmm-fn-of-beta}
\end{figure}

Let us now analyse the behavior of this system mathematically. Our argument goes as follows: First, we assume that the data can be split into disconnected regions. Then it might be favorable that the encoder computes latent codes such that each region covers the full latent space. This means that each latent code $z$ is assigned once in each region. This is a valid encoder function $f_\theta$ and we compute its loss $\loss^{f^{-1}}$ in \cref{appthm:partitions}. In \cref{appcor:beta-crit}, we show that when $\beta < \beta_\text{crit}$ solutions which are not globally invertible have the lowest loss. It is thus vital that $\beta > \beta_\text{crit}$ or larger to ensure the solution is globally invertible.

\begin{figure}
    \centering
    \includegraphics[width=1\linewidth]{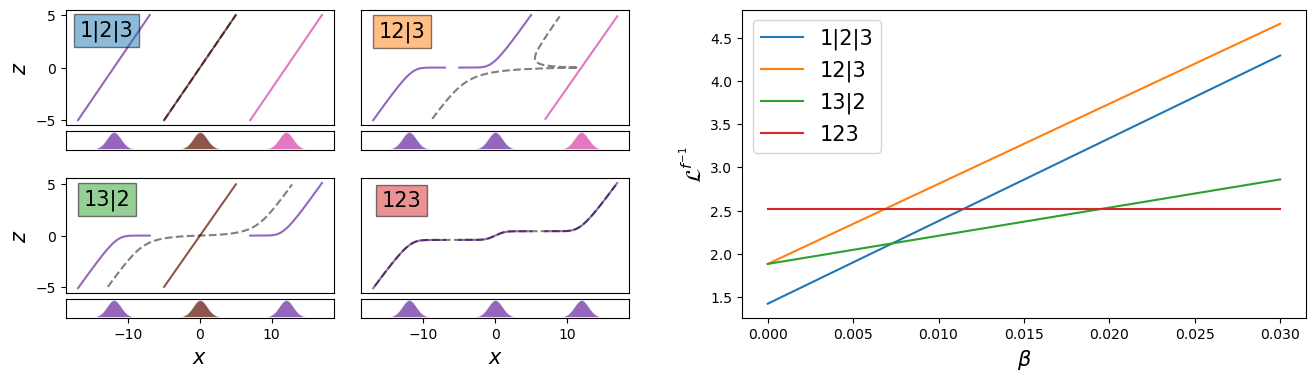}
    \caption{Intuition behind theorem \cref{appthm:partitions}: Comparison of invalid solutions to learning a Gaussian mixture of three modes with non-invertible encoders (blue, orange, green), compared to an invertible encoder (red). \textit{(Left)} As the encoder is not invertible by construction, it may learn to reuse each latent code $z$ once for each disconnected component. This reduces the negative-likelihood at each point, as the derivative $f_\theta'(x)$ is larger at each data point. The decoder (dotted gray line) then cannot reconstruct the data. \textit{(Right)} Increasing $\beta$ increases the importance of reconstruction over maximum likelihood and thus selects the best solution (red).
    }
    \label{fig:partitions-3-gaussians}
\end{figure}

\begin{theorem}
\label{appthm:partitions}
Let $f_\theta$ and $g_\phi$ be $C^1$. Suppose that $q(x)$ may not have support everywhere and allow $f_\theta$ to be non-invertible in the regions where $q(x) = 0$. Suppose the set $\mathcal{S} = \{x: q(x) > 0\}$ is made up of $k$ disjoint, connected components: $\mathcal{S} = \bigcup_{i=1}^k \mathcal{S}_i$. 

For each partition $\mathcal{P}$ of $\{\mathcal{S}_i\}_{i=1}^k$ consider solutions of $\loss^{f^{-1}}$ where
\begin{enumerate}
    \item $f_\theta$ transforms each element of the partition to $p(z)$ individually, and
    \item $g_\phi$ is chosen (given $f_\theta$) such that $R_{\min} = \E_x \left[ \lVert g_\phi(f_\theta(x)) - x \rVert^2 \right]$ is minimal 
\end{enumerate}
The loss achieved is
\begin{equation}
    \loss^{f^{-1}} = h(q(x)) - H(\mathcal{P}) + \beta R_{\min}(\mathcal{P})
\end{equation}
where $h(q(x))$ is the differential entropy of the data distribution and $H(\mathcal{P})$ is the entropy of $\alpha$ where $\alpha_i = \int_{\mathcal{P}_i} q(x) \mathrm{d}x$.
\end{theorem}
Note that the solutions in \cref{appthm:partitions} are not necessarily minima of $\loss^{f^{-1}}$, they just demonstrate what values it can take.

\begin{proof}
Let $\loss = \E_v[\loss^{f^{-1}}]$. The loss can be split into negative log-likelihood and reconstruction parts: $\loss = \loss_\text{NLL} + \beta \loss_\text{R}$. 

Consider a partition $\mathcal{P}$ of $\{\mathcal{S}_i\}_{i=1}^k$.
Let $q_i(x)$ be the distribution which is proportional to $q(x)$ when $x \in \mathcal{P}_i$ but zero otherwise (weighted to integrate to 1):
\begin{equation}
    q_i(x) = \frac{1}{\alpha_i} q(x) 1\!\!1_{x \in \mathcal{P}_i}
\end{equation}with 
\begin{equation}
    \alpha_i = \int_{\mathcal{P}_i} q(x) \mathrm{d}x
\end{equation}
The type of solution described in the theorem statement will be such that $p_\theta(x) = q_i(x)$ for $x \in \mathcal{P}_i$. This means that
\begin{align}
    \loss_\text{NLL} 
        &= -\int q(x) \log p_\theta(x) \mathrm{d}x \\
        &= -\sum_i \int_{\mathcal{P}_i} q(x) \log p_\theta(x) \mathrm{d}x \\
        &= -\sum_i \alpha_i \int_{\mathcal{P}_i} q_i(x) \log q_i(x) \mathrm{d}x \\
        &= \sum_i \alpha_i h(q_i(x))
\end{align}
We also have
\begin{align}
    h(q(x)) 
        &= -\sum_i \int_{\mathcal{P}_i} q(x) \log q(x) \mathrm{d}x \\
        &= -\sum_i \alpha_i \int_{\mathcal{P}_i} q_i(x) \log (\alpha_i q_i(x)) \mathrm{d}x \\
        &= \sum_i \alpha_i \left( h(q_i(x)) - \log \alpha_i \right) \\
        &= \loss_\text{NLL} + H(\alpha)
\end{align}
and therefore
\begin{equation}
    \loss_\text{NLL} = h(q(x)) - H(\mathcal{P})
\end{equation}
Clearly $\loss_\text{R} = R_{\min}(\mathcal{P})$. As a result
\begin{equation}
    \loss = h(q(x)) - H(\mathcal{P}) + \beta R_{\min}(\mathcal{P})
\end{equation}
\end{proof}

\begin{corollary}
\label{appcor:beta-crit}
Call the solution where $\mathcal{P} = \mathcal{S}$ the globally invertible solution. For this solution, $\loss^{f^{-1}} = h(q(x))$.

For a given partition $\mathcal{P}$, the corresponding solution described in \cref{appthm:partitions} has lower loss than the globally invertible solution when $\beta < \beta_\text{crit}$ where
\begin{equation}
    \beta_\text{crit} = \frac{H(\mathcal{P})}{R_{\min}(\mathcal{P})}
\end{equation}
\end{corollary}

\begin{proof}
If $\mathcal{P} = \mathcal{S}$ then $\alpha = (1)$. Therefore $H(\mathcal{P}) = 0$. Since $f$ is invertible in this case, $R_{\min}(\mathcal{P}) = 0$. Therefore $\loss = h(q(x))$.

Now consider a partition $\mathcal{P} \neq \mathcal{S}$. This has loss
\begin{equation}
    \loss = h(q(x)) - H(\mathcal{P}) + \beta R_{\min}(\mathcal{P})
\end{equation}
By solving:
\begin{equation}
    h(q(x)) - H(\mathcal{P}) + \beta R_{\min}(\mathcal{P}) \leq h(q(x))
\end{equation}
we find
\begin{equation}
    \beta \leq \beta_\text{crit} = \frac{H(\mathcal{P})}{R_{\min}(\mathcal{P})}
\end{equation}
\end{proof}

\Cref{appcor:beta-crit} tells us that $\beta$ must be large enough or the minima of $\loss^{f^{-1}}$ will favor solutions which are not globally invertible. In practice, it is difficult to compute the value of $\beta_\text{crit}$ for a given partition, as well as finding the partitions in the first place, so $\beta$ must be tuned as a hyperparameter until a suitable value is found (see \cref{app:reconstruction-weight}). Note that $\beta > \beta_\text{crit}$ does not guarantee that the solution will be globally invertible and globally-invertible solutions may only be the minima of $\loss^{f^{-1}}$ in the limit $\beta \rightarrow \infty$. However, for practical purposes a large value of $\beta$ will be sufficient to get close to the globally invertible solution.

Various solutions for a three-component Gaussian mixture distribution are illustrated in \cref{fig:partitions-3-gaussians}, along with the loss values as a function of $\beta$. Here we approximate regions five or more standard deviations away from the mean as having zero density, in order to partition the space into three parts as per \cref{appthm:partitions}. We see that each solution has a region of lower loss than the globally-invertible solution when $\beta < \beta_\text{crit}$ and that $\beta$ must at least be greater than the largest $\beta_\text{crit}$ (and potentially larger) in order to avoid non-globally-invertible solutions.

While this analysis is for $\loss^{f^{-1}}$, the main conclusion carries over to $\loss^g$, namely that $\beta$ must be sufficiently large to ensure global invertibility. When optimizing $\loss^g$, large $\beta$ is especially important since the loss relies on the approximation $g_\phi \approx f_\theta^{-1}$ which is only achievable if $f_\theta$ is globally invertible.

\section{PRACTICAL GUIDE TO FREE-FORM FLOWS}
\label{app:practical-guide}

This section gives a brief overview over how to get started with adapting free-form flows to a new problem.

\subsection{Model setup}
\label{app:model-setup}

The pair of encoder $f_\theta: \R^D \to \R^D$, which represents $z = f_\theta(x)$, and decoder $g_\phi: \R^D \to \R^D$, which represents $x = g_\phi(z)$, can be any pair of dimension-preserving neural networks. Any architecture is allowed. While in principle batch-norm violates the assumptions for our theorems (because the Jacobians of each item in the batch should be independent), this works well in practice. In our experiments, we found best performance when encoder and decoder each have a global skip connection:
\begin{align}
    z &= f_\theta(x) = x + \tilde f_\theta(x) \\
    x &= g_\phi(z) = z + \tilde f_\theta(z).
\end{align}
This has the advantage that the network is initialized close to the identity, so that training starts close to the parameters where $x \approx g_\phi(f_\theta(x))$ and the reconstruction loss is already low.

\paragraph{Conditional distributions}

If the distribution to be learned should be conditioned on some context $c \in \R^C$, i.e.~$p(x|c)$, feed the context as an additional input to both encoder $f_\theta: \R^D \times \R^C \to \R^D$ and decoder  $g_\phi: \R^D \times \R^C \to \R^D$. For networks with a skip connection:
\begin{align}
    z &= f_\theta(x; c) = x + \tilde f_\theta(x; c) \\
    x &= g_\phi(z; c) = z + \tilde g_\phi(z; c).
\end{align}
If they are multi-layer networks, we observe training to be accelerated when not only the first layer, but also subsequent layers get the input.

\subsection{Training}
\label{app:training}

The PyTorch code in \cref{lst:fff-pytorch} computes the gradient of free-form flows using backward autodiff. The inputs \texttt{encode} and \texttt{decode} can be arbitrary PyTorch functions.

\begin{lstlisting}[float, language=Python, caption=PyTorch implementation of FFF gradient computation, label=lst:fff-pytorch]
import torch
from math import sqrt, prod


def change_of_variables_surrogate(x: torch.Tensor, encode, decode):
    """
    Compute the per-sample surrogate for the change of variables gradient.

    Args: see below
    Returns:
        z: Latent code. Shape: (batch_size, *z_shape)
        x1: Reconstruction. Shape: (batch_size, *x_shape)
        Per-sample surrogate. Shape: (batch_size,)
    """
    x.requires_grad_()
    z = encode(x)

    # Sample v from sphere with radius sqrt(total_dim)
    batch_size, total_dim = x.shape[0], prod(x.shape[1:])
    v = torch.randn(batch_size, total_dim, device=x.device, dtype=x.dtype)
    v *= sqrt(total_dim) / torch.sum(v ** 2, -1, keepdim=True).sqrt()
    v = v.reshape(x.shape)

    # $ g'(z) v $ via forward-mode AD
    with torch.autograd.forward_ad.dual_level():
        dual_z = torch.autograd.forward_ad.make_dual(z, v)
        dual_x1 = decode(dual_z)
        x1, v1 = torch.autograd.forward_ad.unpack_dual(dual_x1)

    # $ v^T f'(x) $ via backward-mode AD
    v2, = torch.autograd.grad(z, x, v, create_graph=True)

    # $ v^T f'(x) stop_grad(g'(z)) v $
    surrogate = torch.sum((v2 * v1.detach()).reshape(batch_size, total_dim), -1)

    return z, x1, surrogate


def fff_loss(x: torch.Tensor, encode, decode, beta: float):
    """
    Compute the per-sample MLAE loss for a latent normal distribution
    Args:
        x: Input data. Shape: (batch_size, *x_shape)
        encode: Encoder function. Takes an input `x` and returns a latent code `z` of shape (batch_size, *z_shape).
        decode: Decoder function. Takes a latent code `z` and returns a reconstruction `x1` of shape (batch_size, *x_shape).
        beta: Weight of the reconstruction error.

    Returns:
        Per-sample loss. Shape: (batch_size,)
    """
    z, x1, surrogate = change_of_variables_surrogate(x, encode, decode)
    nll = torch.sum((z ** 2).reshape(x.shape[0], -1) ** 2).sum(-1) - surrogate
    return nll + beta * ((x - x1).reshape(x.shape[0], -1) ** 2).sum(-1)
\end{lstlisting}

\subsection{Likelihood estimation}
\label{app:likelihood}

For a trained free-form flow, we are interested in how well the learnt model captures the original distribution. We would like to ask ``How likely is our model to generate this set of data?'' We can answer this question via the negative log-likelihood NLL, which is smaller the more likely the model is to generate these data points:
\begin{equation}
    \text{NLL} = -\sum_{i=1}^{N_\text{unseen}} \log p_\phi(X=x_i).
\end{equation}

For normalizing flows with analytically invertible encoder $f_\theta$ and decoder $g_\theta$, evaluating the NLL can be achieved via the change of variables of the encoder, as the encoder Jacobian determinant is exactly the inverse of the decoder Jacobian determinant:
\begin{align}
    -\log p_\theta(X=x_i) &
    = -\log p(Z=g_\theta^{-1}(x)) + \log \det g_\theta'(g_\theta^{-1}(x)) \\&
    = -\log p(Z=f_\theta(x)) - \log \det f_\theta'(x).
\end{align}
The FFF encoder and decoder are only coupled via the reconstruction loss, and the distribution of the decoder (the actual generative model) might be slightly different from the encoder. We therefore compute the change of variables with the decoder Jacobian. In order to get the right latent code that generated a data point, we use the encoder $f_\theta(x)$:
\begin{align}
    -\log p_\phi(X=x_i) &
    = -\log p(Z=g_\phi^{-1}(x)) + \log \det g_\phi'(g_\phi^{-1}(x)) \\&
    \approx -\log p(Z=f_\theta(x)) + \log \det g_\phi'(f_\theta(x)).
    \label{eq:fff-cov}
\end{align}
This approximation $f_\theta(x) \approx g_\phi^{-1}(x)$ is sufficiently valid in practice. For example, for the Boltzmann generator on DW4, we find that the average distance between an input $x$ and its reconstruction $x' = g_\phi(f_\theta(x))$ is $0.0253$. Comparing the energy $u(x)$ to the energy $u(x')$ of the reconstruction, the mean absolute difference is $0.11$, which is less than 1\% of the energy range $\max_{x \in \mathcal X_\text{test}} u(x) - \min_{x \in \mathcal X_\text{test}} u(x) = 13.7$. %

\subsection{Determining the optimal reconstruction weight}
\label{app:reconstruction-weight}

Apart from the usual hyperparameters of neural network training such as the network architecture and training procedure, free-form flows have one additional hyperparameter, the reconstruction weight $\beta$. We cannot provide a rigorous argument for how $\beta$ should be chosen at this stage.

However, we find that it is easy to tune in practice by monitoring the training negative log-likelihood over the first epoch (see \cref{eq:fff-cov}). This involves computing the Jacobian $f_\theta'(x)$ explicitly. We can then do an exponential search on $\beta$:
\begin{enumerate}
    \item If the negative log-likelihood is unstable (i.e.~jumping values; reconstruction loss typically also jumps), increase $\beta$ by a factor.
    \item If the negative log-likelihood is stable, we are in the regime where training is stable but might be slow. Try decreasing $\beta$ to see if that leads to training that is still stable yet faster.
\end{enumerate}
For a rough search, it is useful to change $\beta$ by factors of 10. We observe that there usually is a range of more than one order of magnitude for $\beta$ where the optimization converges to the same quality. We find that training with larger $\beta$ usually catches up with low $\beta$ in late training. Higher $\beta$ also ensures that the reconstruction loss is lower, so that likelihoods are more faithful, see \cref{app:likelihood}.

\section{EXPERIMENTS}
\label{app:experiments}

All our experiments can be reproduced via our public repository at \url{https://github.com/vislearn/FFF/}.

\subsection{Simulation-Based Inference}
\label{app:sbi}

Our models for the SBI benchmark use the same ResNet architecture as the neural spline flows \cite{durkan2019neural} used as the baseline. It consists of 10 residual blocks of hidden width 50 and ReLU activations. Conditioning is concatenated to the input and additionally implemented via GLUs at the end of each residual block. We also define a simpler, larger architecture which consists of 2x256 linear layers followed by 4x256 residual blocks without GLU conditioning. We denote the architectures in the following as ResNet S and ResNet L. To find values for architecture size, learning rate, batch size and $\beta$ we follow \cite{wildberger2023flow} and perform a grid search to pick the best value for each dataset and simulation budget. As opposed to \cite{wildberger2023flow} we run the full grid, but with greatly reduced search ranges, which are provided in Table \ref{tab:sbi-grid}, which amounts to a similar budget. The best hyperparameters for each setting are shown in Table \ref{tab:sbi-architecture}. Notably, this table shows that our method oftentimes works well on the same datasets for a wide range different $\beta$ values. The entire grid search was performed exclusively on compute nodes with ``AMD Milan EPYC 7513 CPU'' resources and took $\sim14.500 \text{h} \times 8 \text{ cores}$ total CPU time for a total of 4480 runs.

\begin{table}[]
    \centering
    \begin{tabular}{c|c}
          & Range \\
         \hline
         Reconstruction weight $\beta$ & $10, 25, 100, 500$ \\
         Learning rate & $\{ 1, 2, 5, 10 \} \times 10^{-4}$ \\
         Batch size & $2^2, ..., 2^8$ \\
         Architecture size & S, L$^{*}$ \\
    \end{tabular}
    \caption{Hyperparameter ranges for the grid search on the SBI benchmark. $^{*}$We only perform the search over architecture size for the 100k simulation budget scenarios.}
    \label{tab:sbi-grid}
\end{table}

\begin{table}[]
    \centering
    \begin{tabular}{l|cccc}
         Dataset & Batch size & Learning rate & $\beta$ & ResNet size \\
         \hline
         Bernouli glm & 8/32/128 & $5 \times 10^{-4}$ & 25/25/500 & S/S/L \\
         Bernouli glm raw & 16/64/32 & $5/10/10 \times 10^{-4}$ & 25/25/50 & S/S/L \\
         Gaussian linear & 8/128/128 & $5/10/1 \times 10^{-4}$ & 25/500/500 & S \\
         Gaussian linear uniform & 8/8/32 & $5/2/5 \times 10^{-4}$ & 500/10/100 & S/S/L \\
         Gaussian mixture & 4/16/32 & $5/2/10 \times 10^{-4}$ & 10/500/25 & S/S/L \\
         Lotka Volterra & 4/32/64 & $10/10/5 \times 10^{-4}$ & 500/500/25 & S \\
         SIR & 8/32/64 & $10/10/5 \times 10^{-4}$ & 500/25/25 & S \\
         SLCP & 8/32/32 & $5 \times 10^{-4}$ & 10/25/25 & S/S/L \\
         SLCP distractors & 4/32/256 & $5/10/5 \times 10^{-4}$ & 25/10/10 & S \\
         Two moons & 4/16/32 & $5/5/1 \times 10^{-4}$ & 500 & S/S/L \\
    \end{tabular}
    \caption{Hyperparamters found by the grid search for the SBI benchmark. Cells are split into the hyperparameters for all three simulation budgets, unless we use the same setting across all of them.}
    \label{tab:sbi-architecture}
\end{table}

\subsection{Molecule Generation}
\label{app:molecule-generation}

\subsubsection{$E(n)$-GNN}
\label{app:en-gnn}

For all experiments, we make use of the $E(n)$ equivariant graph neural network proposed by \cite{satorras2021equivariant} in the stabilized variant in \cite{satorras2021equivariantflow}. It is a graph neural network that takes a graph $(V, E)$ as input. Each node $v_i \in V$ is the concatenation of a vector in space $x_i \in \R^n$ and some additional node features $h_i \in \R^h$. The neural network consists of $L$ layers, each of which performs an operation on $v^l = [x_i^{l}; h_i^{l} \to x_i^{l+1}; h_i^{l+1}]$. Spatial components are transformed \textit{equivariant} under the Euclidean group $E(n)$ and feature dimensions are transformed \textit{invariant} under $E(n)$.
\begin{align}
\mathbf{m}_{i j} & =\phi_e\left(\boldsymbol{h}_i^l, \boldsymbol{h}_j^l, d_{i j}^2, a_{i j}\right), \\ 
\tilde e_{ij} &= \phi_{inf}(m_{ij}), \\ 
\boldsymbol{h}_i^{l+1}&=\phi_h\left(\boldsymbol{h}_i^l, \sum_{j \neq i} \tilde{e}_{i j} \mathbf{m}_{i j}\right), \\
\boldsymbol{x}_i^{l+1} & =\boldsymbol{x}_i^l+\sum_{j \neq i} \frac{\boldsymbol{x}_i^l-\boldsymbol{x}_j^l}{d_{i j}+1} \phi_x\left(\boldsymbol{h}_i^l, \boldsymbol{h}_j^l, d_{i j}^2, a_{i j}\right)
\end{align}
Here, $d_{ij} = \| x_i^l - x_j^l \|$ is the Euclidean distance between the spatial components, $a_{ij}$ are optional edge features that we do not use. The $\tilde e_{ij}$ are normalized for the input to $\phi_h$. The networks $\phi_e, \phi_{inf}, \phi_h, \phi_x$ are learnt fully-connected neural networks applied to each edge or node respectively.

\subsubsection{Latent distribution} 

As mentioned in \cref{sec:experiments}, the latent distribution must be invariant under the Euclidean group. While rotational invariance is easy to fulfill, a normalized translation invariant distribution does not exist. Instead, we adopt the approach in \cite{kohler2020equivariant} to consider the subspace where the mean position of all atoms is at the origin:  $\sum_{i=1}^N x_i = 0$. We then place a normal distribution over this space. By enforcing the output of the $E(n)$-GNN to be zero-centered as well, this yields a consistent system. See \cite{kohler2020equivariant} for more details.

\subsubsection{Boltzmann Generators on DW4, LJ13 and LJ55}
\label{app:bg}

We consider the two potentials
\begin{align}
    \text{Double well (DW):} \qquad v_{\mathtt{DW}}(x_1, x_2) &= \frac1{2\tau} \left( a (d - d_0) + b(d - d_0)^2 + c(d - d_0)^4 \right), \\
    \text{Lennard-Jones (LJ):} \qquad v_{\mathtt{LJ}}(x_1, x_2) &= \frac\epsilon{2\tau} \left( \left(\frac{r_m}{d}\right)^{12} - 2 \left(\frac{r_m}{d}\right)^6 \right).
\end{align}
Here, $d = \| x_1 - x_2 \|$ is the Euclidean distance between two particles. The DW parameters are chosen as $a = 0, b=-4, c=0.9, d_0 = 4$ and $\tau = 1$. For LJ, we choose $r_m = 1, \epsilon=1$ and $\tau = 1$. This is consistent with \citep{klein2023equivariant} and we use their MCMC samples as data, of which we use 400k samples for validation and 500k for testing the final model.

We give hyperparameters for training the models in \cref{tab:bg-hyperparameters}. We consistently use the Adam optimizer. While we use the $E(n)$-GNN as our architecture, we do not make use of the features $h$ because the Boltzmann distributions in question only concern positional information.
Apart from the varying layer count, we choose the following $E(n)$-GNN model parameters as follows: Fully connected node and edge networks (which are invariant) have one hidden layers of hidden width 64 and SiLU activations. Two such invariant blocks are executed sequentially to parameterize the equivariant update. We compute the edge weights $\tilde e_{ij}$ via attention. Detailed choices for building the network can be determined from the code in \cite{hoogeboom2022equivariant}.

\begin{table}
    \centering
    \begin{tabular}{c|ccc}
        & DW4& LJ13& LJ55\\ \hline
        Layer count&  20&  8&  10\\
        Reconstruction weight $\beta$&  10&  200&  500\\
        Learning rate&  0.001&  0.001&  0.001\\
        Learning rate scheduler& One cycle & - & - / Exponential $\gamma=0.99999$\\
        Gradient clip& 1& 1&0.1\\
        Batch size& 256& 256&48 \\
        Duration & 50 epochs & 400 epochs & 300k / 450k steps \\
    \end{tabular}
    \caption{Hyperparameters used for the Boltzmann generator tasks. The format ``A / B`` specifies a two-step training.}
    \label{tab:bg-hyperparameters}
\end{table}

\begin{table}
    \centering
    \begin{tabular}{c|ccc} 
        & \multirow{2}{*}{NLL ($\downarrow$)} & \multicolumn{2}{c}{Sampling time ($\downarrow$)} \\
        && Raw & incl.~$\log q_\theta(x)$ \\ \hline 
        & \multicolumn{3}{c}{DW4} \\ \hline
        $E(n)$-NF & 1.72 $\pm$ 0.01 &0.024 ms & 0.10 ms\\ %
        OT-FM & 1.70 $\pm$ 0.02  &0.034 ms& 0.76 ms\\ %
        E-OT-FM & 1.68 $\pm$ 0.01  &0.033 ms& 0.75 ms\\ %
        FFF & 1.68 $\pm$ 0.01&0.026 ms& 0.74 ms \\ %
        \hline
        
        & \multicolumn{3}{c}{LJ13} \\ \hline
        $E(n)$-NF & -16.28 $\pm$ 0.04 &0.27 ms & \textbf{1.2} ms\\ %
        OT-FM & -16.54 $\pm$ 0.03 &0.77 ms & 38 ms\\ %
        E-OT-FM & -16.70 $\pm$ 0.12 &0.72 ms& 38 ms\\ %
        FFF & \textbf{-17.09 $\pm$ 0.16}&\textbf{0.11} ms& 3.5 ms\\ %
        \hline
        
        & \multicolumn{3}{c}{LJ55} \\ \hline
        OT-FM & -88.45 $\pm$ 0.04   & 40 ms & 6543 ms\\ %
        E-OT-FM & \textbf{-89.27 $\pm$ 0.04} &40 ms & 6543 ms\\ %
        FFF & -88.72 $\pm$ 0.16 &\textbf{2.1} ms& \textbf{311} ms\\ %
    \end{tabular}
    \caption{Boltzmann generator negative log-likelihood and sampling times, including the time to compute the density from the network Jacobians. Note that in all cases, the log prob could be distilled by a $E(3)$-invariant network with scalar output for faster density estimation. NLLs are due to \cite{klein2023equivariant}. Errors are the standard deviations over runs. The other models are based on an ODE trained via maximum likelihood ($E(n)$-NF, \cite{satorras2021equivariantflow}), and trained via Optimal Transport Flow Matching with (OT-FM) or without (E-OT-FM) equivariance-aware matching \citep{klein2023equivariant}. $E(n)$-NF is too memory intensive to train on LJ55 efficiently. Expands \cref{tab:boltzmann-overview} in main text.}
    \label{tab:bg-sampling-time}
\end{table}

\subsubsection{QM9 Molecule Generation}
\label{app:qm9}

For the QM9 \citep{ruddigkeit2012, ramakrishnan2014} experiment, we again employ a $E(3)$-GNN. This time, the dimension of node features $h$ is composed of a one-hot encoding for the atom type and an ordinal value for the atom charge. Like \cite{satorras2021equivariantflow}, we use variational dequantization for the ordinal features \citep{ho2019flow}, and argmax flows for the categorical features \citep{hoogeboom2021argmax}. For QM9, the number of atoms may differ depending on the input. We represent the distribution of molecule sizes as a categorical distribution.

We again employ the $E(3)$-GNN with the same settings as for the Boltzmann generators. We use 16 equivariant blocks, train with Adam with a learning rate of $10^{-4}$ for 700 epochs. We then decay the learning rate by a factor of $\gamma=0.99$ per epoch for another 100 epochs. We set reconstruction weight to $\beta=2000$. We use a batch size of 64.

For both molecule generation tasks together, we used approximately 6,000 GPU hours on an internal cluster of NVIDIA A40 and A100 GPUs. A full training run on QM9 took approximately ten days on a single such GPU.

\subsection{Ablation Studies}
\label{app:ablation}

We study the effect of different modifications to our method in an ablation study on the MINIBOONE dataset shown in Table \ref{tab:miniboone}. Firstly, we train a classical normalizing flow, implemented as a coupling flow \cite{dinh2015nice} (Classical NF) and compare it to a model where the exact likelihood loss of \cref{eq:maximum-likelihood} has been replaced by the trace estimator of \cref{eq:loss-ML-f-inv}, but still using the exact inverse (NF + trace estimator). The resulting NLL shows that using \cref{eq:loss-ML-f-inv} provides a good estimate for the maximum likelihood objective, with only small deterioration due to the increased stochasticity. Next we train two identical coupling flow networks as FFF, meaning we no longer use their invertibility and instead rely on the reconstruction loss of \cref{eq:loss-f-inv-and-g} to learn an inverse (INN as FFF). This shows that the coupling flow architecture is suboptimal as FFF, despite a guarantee of invertibility. Finally, we show that with the right architecture (ResNet as FFF/Transformer as FFF) we can reach and even outperform coupling flows. We also include an ablation study on the effect of increasing the the number of Hutchinson samples $K$ vs. increasing batch size $bs$. Both measures reduce stochasticity of the optimizer and can lead to better performance, but we find that the effect of increasing batch size is more pronounced, and can lead to more improvements than increasing $K$ at the same cost.

\begin{table}
    \centering
    \begin{tabular}{c|c}
          Setup&  NLL\\\hline\hline
          Classical NF (\cref{eq:maximum-likelihood})& 
    10.55\\
  NF + trace estimator (\cref{eq:loss-ML-f-inv})&10.60\\
  \hline
  INN as FFF & 16.42 \\
 Transformer as FFF & 10.10 \\
  \hline
 ResNet as FFF, $K=1, bs=256$& 10.60 \\
 ResNet as FFF, $K=2, bs=256$& 10.18 \\
 ResNet as FFF, $K=1, bs=64$& 12.96 \\
 ResNet as FFF, $K=2, bs=64$& 12.30 \\
 
 \end{tabular}
    \caption{Ablation in terms of NLL on MINIBOONE: We start with a classical normalizing flow trained with exact likelihood, then add the trace estimator. Next, we compare different architectures trained as free-form flows. Finally, we compare using more Hutchinson samples $K$ and varying batch size $bs$.}
\label{tab:miniboone}
\end{table}

\subsubsection{Software libraries}

We build our code upon the following python libraries: PyTorch \citep{paszke2019pytorch}, PyTorch Lightning \citep{falcon2019pytorch}, Tensorflow \citep{tensorflow2015-whitepaper} for FID score evaluation, Numpy \citep{harris2020array}, Matplotlib \citep{hunter2007matplotlib} for plotting and Pandas \citep{mckinney2010data,reback2020pandas} for data evaluation.

\end{document}